\pdfoutput=1
\documentclass[11pt]{article}

\usepackage[letterpaper,margin=1in]{geometry}
\usepackage[T1]{fontenc}
\usepackage[utf8]{inputenc}
\usepackage{amsmath}
\usepackage{amsfonts}
\usepackage{amssymb}
\usepackage{amsthm}
\usepackage{float}
\newfloat{algorithm}{tbp}{loa}
\floatname{algorithm}{Algorithm}
\newenvironment{algorithmic}[1][1]{%
  \begin{list}{}{%
    \setlength{\leftmargin}{2.7em}%
    \setlength{\labelwidth}{0pt}%
    \setlength{\labelsep}{0.5em}%
    \setlength{\itemsep}{0.15ex}%
    \setlength{\parsep}{0pt}%
    \setlength{\topsep}{0.4ex}%
  }%
}{\end{list}}
\newcounter{algline}
\newcommand{\State}{\refstepcounter{algline}\item\makebox[1.8em][r]{\scriptsize\thealgline}\quad}
\newcommand{\Require}[1]{\State\textbf{Input:} #1}
\newcommand{\Ensure}[1]{\State\textbf{Output:} #1}
\newcommand{\For}[1]{\State\textbf{for }#1\textbf{ do}}
\newcommand{\EndFor}{\State\textbf{end for}}
\newcommand{\If}[1]{\State\textbf{if }#1\textbf{ then}}
\newcommand{\EndIf}{\State\textbf{end if}}
\newcommand{\Return}[1]{\State\textbf{return }#1}
\newcommand{\Comment}[1]{\hfill\textcolor{gray}{\(\triangleright\) #1}}
\usepackage{booktabs}
\usepackage{multirow}
\usepackage{makecell}
\usepackage{graphicx}
\usepackage{xcolor}
\usepackage{natbib}
\bibpunct[, ]{(}{)}{,}{a}{}{,}
\usepackage[hypertexnames=false]{hyperref}
\hypersetup{
    colorlinks=true,
    linkcolor={red!50!black},
    citecolor={blue!50!black},
    urlcolor={blue!80!black},
    pdftitle={Leveraging Inference-Time Compute for Diffusion Models via Global Scheduling of Denoising Trajectories},
    pdfauthor={Yuan Cao, Yifu Tang, Hangqi Li, Zeyu Zheng}
}

\newtheorem{theorem}{Theorem}
\newtheorem{lemma}{Lemma}
\newtheorem{proposition}{Proposition}
\newtheorem{corollary}{Corollary}
\newtheorem{definition}{Definition}
\newtheorem{assumption}{Assumption}
\newtheorem{remark}{Remark}

\let\amsthmproof\proof
\let\endamsthmproof\endproof
\renewcommand{\proof}[1]{\amsthmproof[#1]}
\renewcommand{\endproof}{\endamsthmproof}
\newcommand{\Halmos}{}

\newcommand{\GAINS}{\textsc{GAINS}}
\newcommand{\E}{\mathbb{E}}
\newcommand{\R}{\mathbb{R}}
\newcommand{\cL}{\mathcal{L}}
\newcommand{\cG}{\mathcal{G}}
\newcommand{\cD}{\mathcal{D}}
\newcommand{\cH}{\mathcal{H}}
\newcommand{\cA}{\mathcal{A}}
\newcommand{\cX}{\mathcal{X}}
\newcommand{\cC}{\mathcal{C}}
\newcommand{\bsigma}{\boldsymbol{\sigma}}

\begin{document}

\title{Leveraging Inference-Time Compute for Diffusion Models via Global Scheduling of Denoising Trajectories}
\author{Yuan Cao$^{1}$ \quad Yifu Tang$^{2}$ \quad Hangqi Li$^{2}$ \quad Zeyu Zheng$^{2}$\thanks{Emails: \texttt{caoyuanleo@stu.pku.edu.cn}; \texttt{yifutang@berkeley.edu}; \texttt{lihq@berkeley.edu}; \texttt{zyzheng@berkeley.edu}.\\ * Corresponding author.}\\[2pt]
\small $^{1}$School of Mathematical Sciences, Peking University\\
\small $^{2}$Department of Industrial Engineering and Operations Research \& BAIR Lab, University of California, Berkeley\\[2pt]
}
\date{}
\maketitle

\begin{abstract}
Diffusion models generate a sample by traversing a denoising trajectory, a sequence of stochastic noise-reduction steps that progressively transforms pure noise into a draw from a target distribution. At deployment time, additional computation can improve the quality of a generated sample without retraining. The mechanism is simple: at each time step, the sampler draws several candidate noise samples, evaluates the resulting predictions with a quality criterion called the verifier, and retains the best candidate at the cost of one network evaluation per candidate. This raises a resource allocation question: given a fixed budget of function evaluations, how should search effort be distributed across the steps of the denoising trajectory? We formulate this question as a computational budget allocation problem and deliver the following analysis and results. First, we show that, to leading order in the step size, the expected gain from evaluating $K$ candidates at a step factorizes into the product of an endogenous, step-specific sensitivity parameter and a universal sample-size factor equal to the expected best of $K$ standard-normal draws. Second, for a fixed sensitivity profile, the optimal allocation solves a separable concave integer program with water-filling structure. At fixed total sensitivity, its advantage over uniform allocation increases with sensitivity dispersion in the majorization order. Third, we prove that when sensitivities vary across problem instances, no adaptive policy can avoid worst-case regret that grows linearly in the trajectory length, which motivates a design that anchors the allocation offline and adapts online only to recover instance-specific slack. We extend the analysis from independent random search to a broader family of local search operators, and we instantiate the framework as an implementable algorithm. Numerical experiments on three families of diffusion samplers demonstrate that the proposed allocation attains the quality of the uniform benchmark with 20 to 50 percent fewer function evaluations.%
\end{abstract}

\medskip
\noindent\textbf{Key words:} computational budget allocation; simulation optimization; diffusion models; water-filling; sequential search; regret lower bound

\section{Introduction}
\label{sec:intro}

Diffusion models have emerged over the past several years as a leading methodology for generating samples from complex, high-dimensional distributions \citep{ho2020denoising,song2019generative,song2021scorebased}. Concrete applications include high-resolution image synthesis \citep{rombach2022high}, waveform audio synthesis \citep{kong2020diffwave}, and molecular conformation generation \citep{xu2022geodiff}. In these applications, deployment often involves generating many samples under a fixed compute budget, making the trade-off between sample quality and inference cost central to deployment. A diffusion model generates a sample by traversing a \emph{denoising trajectory}: starting from pure Gaussian noise, the sampler applies a sequence of $T$ learned noise-reduction steps, each of which invokes one forward pass through a large neural network, until a sample from the target distribution emerges. We call each forward pass through the denoising network a \emph{function evaluation}; standard sampling therefore uses exactly $T$ function evaluations per sample. Recent work shows that spending more than $T$ function evaluations per sample can improve sample quality without retraining the underlying network, a regime commonly referred to as inference-time scaling \citep{ma2024inference,ramesh2025nts}. The mechanism is a form of sequential stochastic search. At each time step, the sampler draws several candidate noise samples instead of one. For each candidate, the denoising network produces a prediction of the eventual sample, and a quality criterion, which we call a \emph{verifier}, scores that prediction. The sampler retains the candidate with the highest score; each additional candidate therefore adds one function evaluation. This mechanism creates two design choices: how to search within a time step and where to place computation along the trajectory.

Within-step search allocates extra network evaluations among candidate noise samples at one time step and retains the candidate with the highest verifier score. Related inference-time methods search over initial or intermediate noise variables \citep{ma2024inference,tang2024dno,qi2024noises,ramesh2025nts}, maintain particle or tree populations \citep{singhal2025fk,kim2025psi,zhang2025classical}, introduce verifier-free search or noise-aware pruning \citep{zhang2025tscend,ttsnap2025}, and adapt computation across denoising steps \citep{kim2025rbf,lee2025abcd,zhang2025adadiff,ye2025tpdm,yeh2025demons,sundaresha2025vt}. These methods compare candidates at different granularities---within a step, across particles, or across complete trajectories. Our analysis applies to procedures that expose a per-step candidate-count decision, and asks the complementary question of how a fixed total budget should be distributed along the trajectory.

We study that allocation question directly: where along the denoising trajectory should a fixed budget be placed? Replacing one candidate noise sample with a better one improves final quality more at some steps than at others, with the important steps varying across model families (Section~\ref{sec:offline}). Common baselines use uniform budgets or fixed windows, while recent online procedures use heuristic schedules. We formulate the allocation problem explicitly: given a total budget of $B$ function evaluations for one trajectory of $T$ steps, choose per-step candidate counts $K_1, \ldots, K_T$ summing to $B$ to maximize the expected verifier score of the final sample.

The problem combines two established resource-allocation perspectives: simulation budget allocation within each step and sequential resource allocation across steps. The per-step subproblem, in which $K$ candidates are evaluated and the best is retained, is a fixed-budget selection problem of the kind studied in ranking and selection \citep{chen2000simulation,kaufmann2016complexity}. The across-step problem, in which a common budget is divided among $T$ heterogeneous stages, is a discrete resource allocation problem with a separable concave objective \citep{ibaraki1988resource}. The resulting system has two structural features. First, the stages are coupled in sequence: the noise candidate selected at step $t$ determines the state from which step $t-1$ proceeds, so search effort expended early in the trajectory alters the distribution of opportunities available later. Second, search value varies across instances, including prompts and initial noise draws, and allocation decisions are made online in a single forward pass, so each step's expenditure is committed before later opportunities are observed. We analyze these features through a small-noise expansion of the leading-order coupling and a regret analysis of online adaptation.

The paper makes four contributions.

First, we formulate inference-time search at two levels: a local operator selects candidates within each step, and a global scheduler allocates the total budget across steps. The per-step budget links the two levels as the selected latent propagates along the trajectory. Standard sampling, best-of-$N$ selection, and recent search procedures arise as special cases.

Second, we develop the allocation theory. A small-noise expansion factorizes the leading-order gain as $\sigma_t a_K$, separating step sensitivity from a universal Gaussian order-statistic factor. The resulting separable concave integer program has a water-filling solution, and its advantage over uniform allocation increases with sensitivity dispersion at fixed total sensitivity. The same structure extends to local search operators through their operator-specific gain sequences (Theorems~\ref{thm:loc-scale}, \ref{thm:offline}, and~\ref{thm:general-gain}; Proposition~\ref{prop:majorization}).

Third, we quantify the value of online adaptation. Every adaptive policy incurs worst-case regret of order $T$ in a sequential allocation game, and the bound transfers to a realizable diffusion class (Theorem~\ref{thm:lower}, Propositions~\ref{prop:transfer} and~\ref{prop:realizable}). On non-adversarial instances, instance-adaptive allocation gains exactly the Jensen gap over the best fixed allocation (Proposition~\ref{prop:jensen}), motivating an offline anchor with online adjustment.

Fourth, we characterize the gain sequences of $\epsilon$-greedy exploration and local perturbation search. Concave gains yield water-filling marginals, while affine gains concentrate discretionary budget on the most sensitive steps. We instantiate these sequences in \GAINS{}, which matches the uniform benchmark with 20--50 percent fewer function evaluations across three sampler families and improves on it at equal budgets (Theorem~\ref{thm:general-gain}).

\subsection{Related Literature}
\label{sec:related}

Our work draws on, and contributes to, several literatures in operations research and machine learning.

\emph{Resource allocation and simulation-based selection.} The surrogate problem from our leading-order analysis is classical: it maximizes a separable concave objective over integer allocations under a budget. \citet{ibaraki1988resource} give a comprehensive treatment, and greedy marginal allocation underlies our water-filling characterization. The within-step task of evaluating $K$ candidates and retaining the best connects to optimal computing budget allocation \citep{chen2000simulation}, fixed-budget best-arm identification \citep{kaufmann2016complexity,lattimore2020bandit}, sequential information collection \citep{frazier2008knowledge}, and costly sequential search \citep{weitzman1979optimal}. In our setting, candidates are stochastic and state-induced, and effort at one step changes later opportunity distributions; the resulting per-step weights therefore arise endogenously as verifier-gradient sensitivities. Our stopping rule uses feedback in place of a reservation-value calculation.

\emph{Sequential allocation and dynamic coupling.} The across-step problem shares its information structure with dynamic and stochastic knapsacks \citep{kleywegt1998dynamic}, bandits with knapsacks \citep{badanidiyuru2018bandits}, and online allocation with budget pacing \citep{balseiro2019learning,balseiro2023best}: a fixed budget is spent irrevocably over ordered heterogeneous opportunities. These models supply the adversarial and stochastic benchmarks used in our regret analysis. Lagrangian decompositions for weakly coupled stochastic dynamic programs \citep{adelman2008relaxations} and restless bandits \citep{whittle1988restless} motivate marginal-threshold views, while serial inspection models capture downstream propagation of earlier choices \citep{raz1986survey}. Our stages combine both features, with the allocation itself perturbing the stochastic dynamics along the sampling trajectory.

\emph{Inference-time scaling of diffusion models.} Within machine learning, a growing body of work improves generation by expending additional computation at inference time, paralleling test-time scaling for large language models \citep{snell2024scaling,brown2024monkeys}. Diffusion and flow-model methods search over initial or intermediate noise, particles, or trees; learn or reuse intermediate scoring signals; prune candidates before full denoising; or reallocate computation across steps \citep{ma2024inference,tang2024dno,qi2024noises,ramesh2025nts,kim2025rbf,singhal2025fk,kim2025psi,zhang2025classical,zhang2025tscend,ttsnap2025,lee2025abcd,zhang2025adadiff,ye2025tpdm,yeh2025demons,sundaresha2025vt}. Our work isolates the allocation layer: it takes the local search operator as given and derives how its candidate budget should be distributed across denoising steps.

\emph{Theoretical analyses of diffusion samplers.} \citet{gao2024reward} develop a continuous-time reinforcement learning framework for reward-directed diffusion and show that the optimal exploration policy has time-varying variance, a structure consistent with the location-scale model we derive; \citet{tang2024contractive} characterize per-step error sensitivities for score estimation, conceptually related to, though distinct from, our verifier-score sensitivities; and \citet{chen2023sampling} establish convergence guarantees for score-based samplers that underpin the stochastic differential equation discretization we employ. Related efforts in operations research adapt the diffusion model itself rather than its sampling budget \citep{liu2026heavytail,jia2024structured}.

\medskip
\noindent\emph{Organization.} The rest of this paper is organized as follows. Section~\ref{sec:formulation} describes the sequential sampling system and states the budget allocation problem. Section~\ref{sec:framework} introduces the two-level decomposition into local search operators and a global scheduler. Section~\ref{sec:analysis} develops the structural theory: the location-scale factorization, the water-filling characterization of the optimal offline allocation, the worst-case limits of adaptive allocation, and the extension to general local search operators. Section~\ref{sec:algorithm} presents the \GAINS{} algorithm. Section~\ref{sec:experiments} reports numerical experiments. Section~\ref{sec:conclusion} concludes. Proofs of all results are provided in Appendices~\ref{app:proofs} and~\ref{app:general}, and Appendix~\ref{app:mdp} presents a Markov decision process formulation that situates our formulation and several existing procedures within a common framework.

\section{Model and Problem Formulation}
\label{sec:formulation}

This section describes the sequential sampling system underlying diffusion models, defines the verifier and the budget of function evaluations, and states the allocation problem studied in the remainder of the paper. Each step advances the current latent with a fresh noise draw and one forward pass through the network. Readers familiar with diffusion models may proceed directly to Section~\ref{sec:nts-problem}.

\subsection{Diffusion Models as Sequential Stochastic Systems}
\label{sec:diffusion-primer}

A diffusion model generates a sample in $\R^d$ over $T$ discrete denoising steps, indexed $t = T, T-1, \ldots, 1$. We write $x_t \in \R^d$ for the system state (the \emph{latent}) at step $t$, with $x_T$ the initial Gaussian noise and $x_0$ the final sample. In the denoising diffusion probabilistic modeling formulation of \citet{ho2020denoising}, a forward process gradually perturbs data with Gaussian noise, and a learned backward process converts noise back into data.
At each step, the sampler combines two operations: the network computes a deterministic denoising update, and the sampler adds a Gaussian perturbation whose scale is set by the noise schedule. The notation below makes these two parts explicit.

Setting $\alpha_t = 1-\beta_t$ and $\bar{\alpha}_t = \prod_{s=1}^{t}\alpha_s$ for a fixed noise schedule $\beta_t\in(0,1)$, a sample is produced by iterating
\begin{equation}
  x_{t-1}
  = \tfrac{1}{\sqrt{1 - \beta_t}}\left(x_t - \tfrac{\beta_t}{\sqrt{1 - \bar{\alpha}_t}}\,\varepsilon_\theta(x_t, t)\right) + \tilde\sigma_t\,\varepsilon_t,
  \qquad \varepsilon_t\sim\mathcal{N}(0, I_d),
  \label{eq:ddpm-reverse}
\end{equation}
from $x_T\sim\mathcal{N}(0, I_d)$ down to $x_0$, where $\varepsilon_\theta$ is a learned noise predictor, $\tilde\sigma_t$ is a fixed reverse-process noise scale, $\varepsilon_t \in \R^d$ is the noise injected at step $t$, and $I_d$ denotes the $d\times d$ identity matrix. Score-based samplers \citep{song2019generative}, flow-based samplers \citep{lipman2023flow,liu2023flow}, and accelerated samplers \citep{song2021ddim,lu2022dpm} share the same structure: a $T$-step backward trajectory driven by per-step stochastic noise $\varepsilon_t$. Two features of this system deserve emphasis for what follows. First, each step consumes one forward pass through a large neural network, which constitutes the dominant computational cost of sampling. Second, the injected noise $\varepsilon_t$ is the sole source of randomness at step $t$, so that intervening on $\varepsilon_t$ is the natural lever for influencing sample quality at generation time.

\subsection{Verifier and Budget}

The quality of a generated sample is measured by an exogenous scoring criterion, which we now define.

\begin{definition}\label{def:verifier} (Verifier.)
Let $\cX \subseteq \R^d$ denote the sample space (for instance, images or latent representations) and $\cC$ the space of conditioning signals (for instance, text prompts or class labels). A \emph{verifier} is a function $v : \cX \times \cC \to \R$ that assigns to each sample $x \in \cX$ a scalar quality score $v(x, c)$ against the conditioning signal $c \in \cC$. The verifier is fixed throughout the search and is treated as a black box; the analysis in Section~\ref{sec:analysis} additionally imposes a smoothness condition on the score map induced by $v$ (Assumption~\ref{asm:smooth}).
\end{definition}

Verifiers of this kind are the diffusion-model counterpart of the outcome verifiers employed in the test-time scaling of large language models \citep{cobbe2021gsm8k,lightman2023verify}; concrete examples used in our numerical experiments include the average luminance of an image and the size of its compressed representation.

\begin{definition}\label{def:nfe} (Budget of function evaluations.)
A single \emph{function evaluation} is one forward pass through the denoising network, the dominant cost in diffusion sampling. Standard sampling consumes exactly $T$ function evaluations. The scheduler receives a total budget of $B \ge T$ function evaluations per sample; the discretionary portion $B - T$ is available for search.
\end{definition}

\subsection{The Noise Trajectory Search Problem}
\label{sec:nts-problem}

We can now state the decision problem. We abstract the one-step transition \eqref{eq:ddpm-reverse} as a map $F_\theta : \R^d \times \{1, \ldots, T\} \times \cC \times \R^d \to \R^d$ that carries $(x_t, t, c, \varepsilon_t)$ to $x_{t-1}$; this abstraction is formalized in Section~\ref{sec:framework}.

\begin{definition}\label{def:nts} (Noise trajectory search problem.)
Given a pretrained diffusion model with one-step transition $F_\theta$, a conditioning signal $c \in \cC$, a verifier $v(\cdot, c)$ as in Definition~\ref{def:verifier}, and a total budget of $B$ function evaluations as in Definition~\ref{def:nfe}, design a sequential search procedure that expends at most $B$ function evaluations and outputs a sequence of noise realizations $\{\varepsilon_t^*\}_{t=1}^T$, so as to maximize the expected verifier score $\E[v(x_0, c)]$ of the resulting final sample.
\end{definition}

Definition~\ref{def:nts} describes a sequential allocation problem: the noise realization selected at step $t$ shapes the state $x_{t-1}$, which in turn determines which noise realizations are effective at subsequent steps $t' < t$. The search proceeds online in a single pass from $t = T$ down to $t = 1$, so expenditure at a step is committed before later opportunities are observed. At each step, the scheduler balances the value of further local search against preserving budget for the remainder of the trajectory. The two-level decomposition of the next section provides the structure through which we analyze this trade-off.

\section{A Two-Level Decomposition}
\label{sec:framework}

This section decomposes the noise trajectory search problem into two coupled components: a \emph{local search operator} $\cL_t$ that explores candidate noise realizations at a single step, and a \emph{global scheduler} $\cG$ that decides how the budget $B$ is divided across the $T$ steps. The local operator updates the latent state passed to later steps, so the components are coupled through the trajectory; the scheduler's control input to the local operator is the per-step budget $K_t$. Consequently, any local operator implementing this interface can be paired with any global schedule. This separation allows us to study the allocation question, which is the focus of this paper, while treating the local search procedure as a modular building block; Figure~\ref{fig:architecture} depicts the architecture.

\subsection{Transition Abstraction}

To separate within-step search from cross-step scheduling, we first write a single reverse step in an abstract, sampler-agnostic form. Every diffusion-family sampler generates a sample by traversing $T$ discrete steps, and at each step applies
\begin{equation}
  x_{t-1} = F_\theta(x_t, t, c, \varepsilon_t),
  \label{eq:transition}
\end{equation}
where $x_t$ is the latent at step $t$, $c$ is the conditioning signal, and $\varepsilon_t$ is the noise variable injecting stochasticity. The final step yields the sample $x_0$, which the verifier scores as $v(x_0, c)$. The abstraction \eqref{eq:transition} is deliberately model-agnostic. For stochastic samplers, including the formulation \eqref{eq:ddpm-reverse}, the variable $\varepsilon_t$ is the injected Gaussian noise. For deterministic samplers such as flow-matching solvers of ordinary differential equations \citep{lipman2023flow}, stochasticity can be introduced by replacing a deterministic step with a stochastic differential equation step of matched marginal distribution \citep{singh2024stochastic,kim2025rbf}, after which the same abstraction applies.

\begin{figure}[t]
\centering
\includegraphics[width=0.58\linewidth]{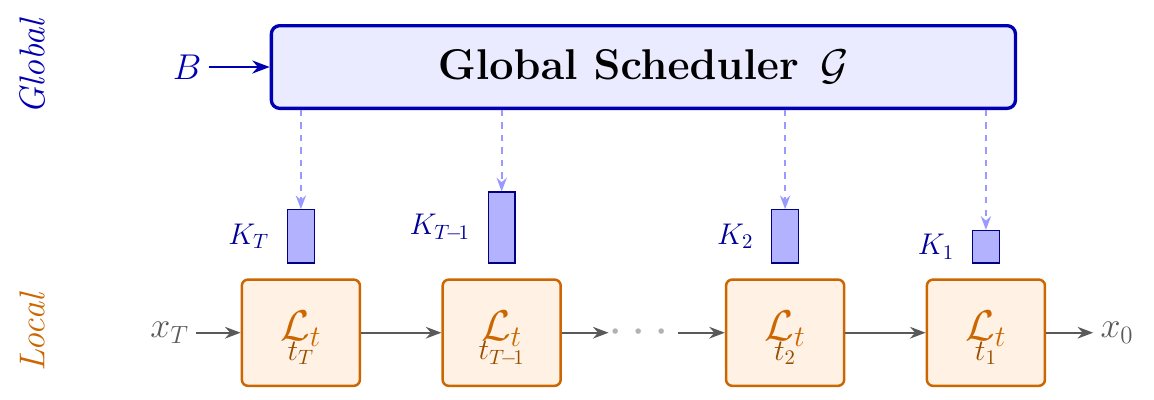}
\caption{The two-level decomposition of noise trajectory search. The global scheduler $\cG$ (top) receives the budget $B$ and distributes per-step budgets $\{K_t\}$ across the trajectory. At each step, the local search operator $\cL_t$ (bottom) evaluates $K_t$ candidate noise realizations and retains the best. The bars depict a non-uniform allocation, in which steps where replacing the noise candidate yields larger quality gains receive larger budgets.}
\label{fig:architecture}
\end{figure}

\subsection{Local Search Operator}
\label{sec:local-operator}

At step $t$, the local operator spends its per-step budget $K_t$ by comparing $K_t$ candidate noise realizations and keeping the best-scoring one. Concretely, given the current state $(x_t, c)$, the verifier $v$, and the transition $F_\theta$, it draws $\{\varepsilon_t^{(k)}\}_{k=1}^{K_t}$, advances each candidate through the transition, and maps each resulting latent to a prediction of the final clean sample by means of the standard one-step denoiser $D_\theta$ \citep{song2021scorebased}:
\begin{equation}
  x_{t-1}^{(k)} = F_\theta(x_t, t, c, \varepsilon_t^{(k)}),
  \qquad
  \hat{x}_0^{(k)} = D_\theta(x_{t-1}^{(k)}, t-1, c).
  \label{eq:x0_prediction}
\end{equation}
The denoiser $D_\theta$, which predicts the final clean sample from any intermediate latent, is distinct from the one-step transition $F_\theta$. The operator computes the score $s^{(k)} = v(\hat{x}_0^{(k)}, c)$ of each candidate and commits to the candidate with the highest score. Each candidate costs one function evaluation, so the operator consumes $K_t$ function evaluations in total. We write $\cL_t : (x_t, c, v, F_\theta, K_t) \mapsto x_{t-1}$ for the resulting map. Many search procedures proposed in the literature, including zero-order perturbation \citep{tang2024dno}, $\epsilon$-greedy exploration \citep{ramesh2025nts}, and particle-based refinement \citep{kim2025rbf}, are instances of this operator; for the purposes of our analysis, $\cL_t$ is a black-box map whose only control variable is the per-step budget $K_t$.

\subsection{Global Scheduler}
\label{sec:global-scheduler}

The local operator determines how to search at a given step; the global scheduler determines where along the trajectory to invest computational effort. A global scheduler
\[
\cG : \bigl(x_T, c, v, F_\theta, B, \{\cL_t\}_{t=1}^T\bigr) \longmapsto \bigl(x_0, \{x_t\}_{t=1}^T\bigr)
\]
walks the steps from $T$ down to $1$ and selects at each step a per-step budget $K_t \ge 1$ subject to $\sum_t K_t \le B$. Setting $K_t = 1$ applies a single base step, whereas $K_t > 1$ invokes the local operator on $K_t$ candidates. The allocation is executed \emph{online}: at step $t$, future search productivity is unknown and expenditure at the current step is committed.

This formulation recovers several existing strategies as special cases corresponding to fixed schedules. Plain diffusion sampling corresponds to $K_t = 1$ for all $t$; best-of-$N$ selection over initial noise corresponds to $K_T = N$ and $K_t = 1$ for $t < T$; and several noise-search and tree-search procedures correspond to particular local operators paired with hard-coded schedules. For policies that switch among verifiers of heterogeneous cost or backtrack to earlier steps, Appendix~\ref{app:mdp} embeds the formulation into a general Markov decision process and classifies ten existing procedures as policies within it.

\section{Analysis of Optimal Budget Allocation}
\label{sec:analysis}

This section develops the structural theory that underlies our allocation procedure. Section~\ref{sec:loc-scale} establishes that the expected gain from local search at a step factorizes into a step-specific sensitivity parameter and a universal sample-size factor. Section~\ref{sec:water-filling} builds on this factorization to characterize the optimal offline allocation. Section~\ref{sec:adaptivity} analyzes the achievable value of online adaptation. Section~\ref{sec:general-operators} extends the theory from independent random search to a general family of local search operators. Throughout Sections~\ref{sec:loc-scale} to~\ref{sec:adaptivity}, the local operator is \emph{random search}, which draws $K$ independent Gaussian noise candidates at each step and retains the best; this instantiation yields the cleanest statements, and Section~\ref{sec:general-operators} shows that the same structure persists for a broad operator family.

\subsection{Location-Scale Structure of Per-Step Gains}
\label{sec:loc-scale}

To analyze the effect of search at a single step, we adopt the standard continuous-time view of diffusion sampling, under which the update \eqref{eq:ddpm-reverse} is represented locally by one Euler-Maruyama step of a reverse-time stochastic differential equation \citep{song2021scorebased}. In plain terms, one small reverse step first follows a deterministic drift and then adds a Gaussian perturbation; the local search varies this perturbation while holding the current state fixed. The discrete update in \eqref{eq:ddpm-reverse} supplies the corresponding terms at the Euler scale: $b_t h$ represents its deterministic denoising change, and $g_t\sqrt h\,\xi$ represents its stochastic term, corresponding to $\tilde\sigma_t\varepsilon_t$ in \eqref{eq:ddpm-reverse}.
\begin{equation}
  X_{t-h} = \bar X_{t-h} + g_t\sqrt h\,\xi,
  \qquad
  \bar X_{t-h} := X_t + b_t(X_t,c)\,h,
  \qquad \xi\sim\mathcal N(0,I_d),
  \label{eq:em-step}
\end{equation}
where $h$ is the integration step size for one discrete solver step, $b_t$ is the drift, $g_t$ is the diffusion coefficient (both determined by the sampler and its noise schedule), $\bar X_{t-h}$ is the noise-free Euler point, and $\xi$ is the per-step noise over which the local operator searches. Let $D_\theta$ denote the one-step denoiser introduced in Section~\ref{sec:local-operator}, define the induced score map $\Psi_t(x;c) := v(D_\theta(x,t{-}h,c);c)$, and write $S_t := \Psi_t(X_{t-h};c)$ for the score of the one-step outcome. The gradient $\nabla\Psi_t$ is taken in the first argument and evaluated at $\bar X_{t-h}$. Before stating the theorem we fix notation, since the letters $g$ and $\sigma$ are each reused: the diffusion coefficient $g_t$ here is distinct from the per-iteration score gain $g_t^{(j)}$ of Section~\ref{sec:algorithm}, and the sensitivity $\sigma_t$ defined below is distinct from both the reverse-process noise scale $\tilde\sigma_t$ in \eqref{eq:ddpm-reverse} and the within-step empirical standard deviation $\sigma$ in Algorithm~\ref{alg:gains}.

Our analysis requires one regularity condition on the induced score map.

\begin{assumption}\label{asm:smooth} (Verifier smoothness.)
For each $t$ and $c$, the map $\Psi_t(\cdot\,;c)$ belongs to $C^2(\R^d)$, and there exists $L_t<\infty$ such that $\|\nabla^2\Psi_t(x;c)\|_{\mathrm{op}}\le L_t$ for all $x \in \R^d$.
\end{assumption}

Assumption~\ref{asm:smooth} requires that the composition of the denoiser and the verifier be twice differentiable with a globally bounded Hessian. The condition holds when both networks employ smooth activation functions with bounded first and second derivatives. For rectified-linear networks or discontinuous verifiers such as compressed file size, the results describe the corresponding smoothed score map. The assumption is imposed on the composite map $\Psi_t$ rather than on the verifier alone, since it is the composite map that the search perturbs. The following result describes the distribution of the one-step score under Assumption~\ref{asm:smooth}.

\begin{theorem}\label{thm:loc-scale} (Location-scale structure of per-step scores.)
Suppose Assumption~\ref{asm:smooth} holds. Define $\mu_t := \Psi_t(\bar{X}_{t-h};c)$ and $\sigma_t := g_t\sqrt{h}\,\|\nabla\Psi_t(\bar{X}_{t-h};c)\|$. Conditional on $(X_t,c)$, the following statements hold.
\begin{enumerate}
\item[(i)] If $\sigma_t>0$, define $u_t := \nabla\Psi_t(\bar X_{t-h};c)/\|\nabla\Psi_t(\bar X_{t-h};c)\|$ and $Z_t := \langle u_t,\xi\rangle$. Then $Z_t\sim\mathcal{N}(0,1)$. If $\sigma_t=0$, choose any standard Gaussian variable for $Z_t$. In both cases,
\[
  S_t = \mu_t + \sigma_t Z_t + R_t,
\]
where $|R_t|\le\tfrac{1}{2}L_t g_t^2 h\|\xi\|^2$ almost surely. The variables $Z_t$ and $R_t$ need not be independent.
\item[(ii)] As $g_t\sqrt{h}\to 0$, $\mathrm{Var}(S_t\mid X_t,c) = \sigma_t^2 + O_{d, L_t, \|\nabla\Psi_t\|}(g_t^3 h^{3/2})$, where the implicit constant depends on $d$, the Hessian bound $L_t$, and the gradient norm at $\bar X_{t-h}$.
\item[(iii)] For $K$ independent noise candidates $\xi^{(k)}\sim\mathcal N(0,I_d)$ at step $t$, let $S_t^{(k)}:=\Psi_t(\bar X_{t-h}+g_t\sqrt h\,\xi^{(k)};c)$. Then the expected gain from retaining the best of $K$ candidates satisfies
\begin{equation}
G_t(K) := \E\Bigl[\max_{1\le k\le K} S_t^{(k)}\,\Big|\,X_t,c\Bigr] - \mu_t = \sigma_t\,a_K + O_{K,d,L_t}(g_t^2 h),
\label{eq:gain-sde}
\end{equation}
where $a_K := \E[\max(Z_1,\ldots,Z_K)]$ for independent standard Gaussian variables $Z_1, \ldots, Z_K$, with $a_K\sim\sqrt{2\log K}$ as $K$ grows.
\end{enumerate}
\end{theorem}

The proof of Theorem~\ref{thm:loc-scale} is provided in Appendix~\ref{app:proofs}; the argument rests on a second-order expansion of the induced score map around the noise-free point, followed by a projection of the isotropic noise onto the verifier-gradient direction, with the properties of $a_K$ supplied by classical order-statistic theory. Before building the allocation theory upon this result, we pause to interpret its content. Part~(i) states that, to leading order in the effective noise magnitude $g_t\sqrt h$, the one-step score is a location-scale transformation of a standard Gaussian variable. The scale parameter
\begin{equation}
  \sigma_t \;=\; g_t\sqrt{h}\,\bigl\|\nabla\Psi_t(\bar{X}_{t-h};c)\bigr\|
  \label{eq:sensitivity}
\end{equation}
is the product of three interpretable quantities: the diffusion coefficient $g_t$, which measures how much randomness the schedule injects at step $t$; the square root of the step size; and the norm of the verifier gradient at the noise-free point, which measures how responsive the final score is to perturbations of the current state. We refer to $\sigma_t$ as the \emph{sensitivity} of step $t$. Steps with strong noise injection or steep verifier response are precisely the steps at which one additional candidate is worth the most, which accords with the empirical observation that search is far more productive at some steps than at others. Part~(ii) shows that the within-step variance of candidate scores estimates $\sigma_t^2$ to leading order; this observation is of practical significance, because it means the sensitivity, an a priori unobservable model quantity, can be estimated from the very scores that the search procedure generates. Part~(iii) delivers the factorization on which the remainder of the paper rests: the expected gain from $K$ candidates separates into the step-specific factor $\sigma_t$ and the universal, step-independent sample-size factor $a_K$. The concavity of the Gaussian order-statistic sequence $a_K$ \citep{david2003order} expresses the diminishing returns of local search, and the logarithmic growth $a_K \sim \sqrt{2\log K}$ quantifies how quickly those returns saturate.

\subsection{Optimal Offline Allocation}
\label{sec:water-filling}

Treating the sensitivities $\{\sigma_t\}$ as a fixed input and dropping the $O(g_t^2 h)$ remainder in \eqref{eq:gain-sde}, we study the fixed-profile leading-order surrogate
\begin{equation}
  \max_{\{K_t\}_{t=1}^T}\;\sum_{t=1}^{T}\sigma_t\,a_{K_t}
  \quad\text{subject to}\;\sum_{t=1}^{T}K_t=B,\;K_t\in\mathbb{Z}_{\ge 1}.
  \label{eq:alloc}
\end{equation}
Problem \eqref{eq:alloc} is a discrete resource allocation problem with a separable concave objective, a structure with a long history in operations research \citep{ibaraki1988resource}. We write $\Delta a_K := a_{K+1}-a_K$ for the discrete marginal gain of the sample-size factor and call step $t$ \emph{active} at an allocation $\{K_t^*\}$ if $K_t^* > 1$.

\begin{remark}\label{rmk:regime} (Validity of the surrogate.)
The remainder in \eqref{eq:gain-sde} carries a constant that grows linearly in $K$, so the aggregate remainder over the trajectory is of order $\max_t d\,L_t\,g_t^2 h \cdot B$, whereas the surrogate objective is $\sum_t \sigma_t a_{K_t}$. Dropping the remainder is therefore justified when the former is dominated by the latter; a sufficient per-step condition is $K_t\,g_t\sqrt h = o\bigl(a_{K_t}\|\nabla\Psi_t\|/(d\,L_t)\bigr)$ for every active step, which covers both fixed budgets with $g_t \sqrt h \to 0$ and budgets growing with $B$ provided $\max_t K_t\,g_t\sqrt h$ vanishes at the stated rate. The optimization statements below are exact for the surrogate problem~\eqref{eq:alloc}; this regime governs only when that surrogate is an accurate approximation to the original gain criterion.
\end{remark}

The following theorem characterizes every optimizer of \eqref{eq:alloc}.

\begin{theorem}\label{thm:offline} (Optimal allocation for the fixed-profile surrogate.)
Let $B\ge T$, and suppose that $\sigma_t$ depends only on the step index $t$, and not on $(c, x_t)$. Then every optimizer $\{K_t^*\}$ of \eqref{eq:alloc} satisfies the following properties.
\begin{enumerate}
\item[(i)] \emph{Instance-independence.} The same allocation $\{K_t^*\}$ optimizes \eqref{eq:alloc} for every $(c, x_t)$.
\item[(ii)] \emph{Common marginal-threshold conditions.} There exists a common marginal threshold $\lambda^* \ge 0$ such that for every active step $t$,
\begin{equation}
  \sigma_t\,\Delta a_{K_t^* - 1} \;\ge\; \lambda^* \;\ge\; \sigma_t\,\Delta a_{K_t^*},
  \label{eq:marginal-threshold}
\end{equation}
while every inactive step satisfies $\sigma_t\,\Delta a_1 \le \lambda^*$.
\item[(iii)] \emph{Monotonicity in sensitivity.} If $\sigma_{t_1} > \sigma_{t_2}$, then $K_{t_1}^* \ge K_{t_2}^*$. Moreover, no optimizer can have two active steps tied at $K_{t_1}^*=K_{t_2}^*=K\ge2$ when $\sigma_{t_1}\Delta a_K>\sigma_{t_2}\Delta a_{K-1}$. For the continuous relaxation based on the interpolant $\tilde a$ constructed in the proof, every relaxed optimizer $\{\widetilde K_t^*\}$ satisfies
\[
  \sigma_{t_1}>\sigma_{t_2}
  \ \text{and}\ 
  \widetilde K_{t_1}^*>1
  \quad\Longrightarrow\quad
  \widetilde K_{t_1}^*>\widetilde K_{t_2}^*.
\]
Thus distinct sensitivities may remain tied in the relaxation only when both coordinates are pinned at the floor.
\item[(iv)] \emph{Dominance over the uniform allocation.} If $T$ divides $B$, so that the uniform allocation $K_t \equiv B/T$ is feasible, then
\begin{equation}
  \sum_{t=1}^T \sigma_t\,a_{K_t^*} \;\ge\; \sum_{t=1}^T \sigma_t\,a_{B/T}.
  \label{eq:dominance}
\end{equation}
Independently of this divisibility condition, consider the associated continuous relaxation with the $C^1$ strictly concave interpolant $\tilde a$ constructed in the proof, whose derivative is positive and locally Lipschitz. The relaxed optimal value strictly exceeds the value of the uniform point whenever $B>T$ and the sensitivities are not all equal. When $B>T$, for any distinct $t_1,t_2$, the objective change generated by the feasible perturbation $K_{t_1}=B/T+\epsilon$, $K_{t_2}=B/T-\epsilon$ is
\[
  \epsilon(\sigma_{t_1}-\sigma_{t_2})\tilde a'(B/T)+O(\epsilon^2)
\]
for all sufficiently small $\epsilon>0$.
\end{enumerate}
\end{theorem}

The proof of Theorem~\ref{thm:offline} is provided in Appendix~\ref{app:proofs}. The integer statements follow from a one-unit exchange argument and admit the familiar water-filling interpretation: an optimal allocation balances sensitivity-weighted discrete marginal gains around a common threshold. Higher-sensitivity steps receive weakly more candidates; equalities can occur because allocations are integral and have a floor. These statements apply exactly to the fixed-profile surrogate~\eqref{eq:alloc}, while Remark~\ref{rmk:regime} specifies the approximation regime for the original criterion. The continuous perturbation in part~(iv) motivates the dispersion analysis that follows.

The first-order statement in part~(iv) admits a global counterpart. Recall that a sensitivity profile $\bsigma = (\sigma_1, \ldots, \sigma_T)$ \emph{majorizes} another profile $\bsigma'$ if, after sorting both in decreasing order, every partial sum of $\bsigma$ weakly dominates the corresponding partial sum of $\bsigma'$, with equality for the total; majorization is the standard partial order of dispersion at fixed total mass \citep{marshall2011inequalities}.

\begin{proposition}\label{prop:majorization} (Dispersion comparative statics.)
Let $B \ge T$, let $\tilde a$ be a $C^1$ strictly concave interpolant of $\{a_K\}$ on $[1,\infty)$ as constructed in the proof of Theorem~\ref{thm:offline}, and define, for $\bsigma \in \R_{+}^T$, the relaxed value $\tilde V^*(\bsigma) := \max\{\sum_t \sigma_t\,\tilde a(K_t) : \sum_t K_t = B,\ K_t \ge 1\}$ and the \emph{uniform surplus} $\Delta_{\mathrm{unif}}(\bsigma) := \tilde V^*(\bsigma) - \tilde a(B/T)\sum_t \sigma_t$. Then the following statements hold. (i) $\Delta_{\mathrm{unif}}$ is nonnegative, convex, and symmetric under permutations of the steps. (ii) $\Delta_{\mathrm{unif}}$ is Schur-convex: if $\bsigma$ majorizes $\bsigma'$, then $\Delta_{\mathrm{unif}}(\bsigma) \ge \Delta_{\mathrm{unif}}(\bsigma')$. (iii) If $B > T$, then $\Delta_{\mathrm{unif}}(\bsigma) = 0$ if and only if all sensitivities are equal.
\end{proposition}

The proof of Proposition~\ref{prop:majorization} is provided in Appendix~\ref{app:proofs}; it combines the representation of $\tilde V^*$ as a pointwise maximum of linear functions with the classical fact that symmetric convex functions are monotone in the majorization order. The proposition upgrades the directional, first-order comparison of Theorem~\ref{thm:offline}(iv) to a global comparative-statics statement at the level of the continuous relaxation: holding total sensitivity fixed, any mean-preserving increase in the dispersion of the profile, in the majorization order, weakly raises the value of optimal reallocation over the uniform benchmark. This comparative-statics result provides a theoretical lens for the cross-model patterns examined in Section~\ref{sec:experiments}, where more dispersed sensitivity profiles are associated with larger gains from global reallocation.

Theorem~\ref{thm:offline} assumes a sensitivity profile common across instances. Under this assumption, the surrogate allocation problem~\eqref{eq:alloc} is solved offline, and part~(i) establishes its instance-independence. In applications, sensitivities also depend on the conditioning signal and the realized trajectory. The next subsection analyzes allocation in this heterogeneous setting.

\subsection{The Value of Adaptive Allocation}
\label{sec:adaptivity}

When the sensitivity profile $\bsigma$ varies across instances, the scheduler can adapt its budget to sensitivities observed so far. We analyze this setting through a sequential allocation game: at each round $t = 1, \ldots, T$, the policy observes the current reward distribution $\cD_t$, which in the diffusion problem is the law of the gain $\sigma_t \Delta a_1$ from one additional candidate; it then chooses $K_t \in \{0,1\}$ subject to $\sum_t K_t \le B$ and collects a reward distributed according to $\cD_t$ if $K_t = 1$ and zero otherwise. Regret is measured against the offline benchmark that knows the entire sequence $\cD_1,\ldots,\cD_T$ in advance.

\begin{theorem}\label{thm:lower} (Worst-case regret lower bound.)
For every, possibly randomized, adaptive policy there exists an instance of the sequential allocation game with $T = 2B$ such that the expected regret satisfies
\begin{equation}
  \E[R_T] \;\ge\; \tfrac{B}{4} \;=\; \Omega(T).
  \label{eq:lower}
\end{equation}
\end{theorem}

The proof of Theorem~\ref{thm:lower} is provided in Appendix~\ref{app:proofs}; it rests on an indistinguishability construction in which two instances agree on the first half of the horizon and diverge on the second, so that any budget committed early is wasted on one instance and any budget conserved is wasted on the other. Because the per-round reward is bounded, the linear order is tight in this adversarial regime: policies incur regret $\Theta(T)$ on the constructed sequence. The theorem applies to policies without prior instance knowledge and motivates the design principle at the end of this subsection.

The game of Theorem~\ref{thm:lower} captures the allocation mechanism. The following proposition embeds its hard profiles into the diffusion problem whenever the family realizes them with uniformly small approximation error, and quantifies the resulting transfer of linear worst-case regret. It uses the extra-budget formulation in which each step receives one mandatory candidate and $e_t := K_t - 1 \ge 0$ discretionary candidates, together with the surrogate objective of Section~\ref{sec:water-filling}.

\begin{proposition}\label{prop:transfer} (Diffusion transfer of the lower bound.)
Fix $B \ge 1$, let $T = 2B$, and consider discretionary allocations $\{e_t\}$ with $\sum_{t=1}^{2B} e_t = B$ evaluated by the surrogate gain $F_{\bsigma}(\mathbf{e}) := \sum_{t=1}^{2B} \sigma_t\,(a_{1+e_t} - a_1)$. Consider the two sensitivity profiles
\[
  \sigma_t^{A} = \begin{cases} m, & t \le B,\\ 0, & t > B,\end{cases}
  \qquad
  \sigma_t^{B} = \begin{cases} m, & t \le B,\\ 2m, & t > B,\end{cases}
\]
with $m > 0$. Let $g_{\max}:=\max_{1\le t\le 2B}g_t$ for the diffusion-transfer statement. Then every, possibly randomized, policy that chooses $e_t$ upon reaching step $t$ on the basis of past observations suffers worst-case surrogate regret at least $m\,\Delta a_1\,B/2$ across the two profiles. Consequently, if there exist two diffusion instances whose gains satisfy $|G_t(K) - \sigma_t a_K| \le C K g_t^2 h$ uniformly for these two profiles, then the realized regret in the diffusion problem is at least $m\,\Delta a_1\,B/2 - O(B g_{\max}^2 h)$.
\end{proposition}

The proof of Proposition~\ref{prop:transfer} is provided in Appendix~\ref{app:proofs}. The proposition embeds the hard sensitivity profiles into the diffusion setting whenever the family realizes them with uniform small-step error, transferring the worst-case regret lower bound up to the stated $O(B g_{\max}^2 h)$ term. It quantifies the implication of adversarial heterogeneity for the configurations studied in Section~\ref{sec:experiments}. Proposition~\ref{prop:realizable} in Appendix~\ref{app:proofs} gives a stylized instance class with driftless linear-Gaussian dynamics and a linear verifier in which the per-step gains equal $\sigma_t a_K$ exactly, so the transfer is unconditional.

Theorem~\ref{thm:lower} identifies the adversarial value of full adaptivity. Our final result of this subsection quantifies the value of instance-adaptive allocation on heterogeneous, non-adversarial instances.

\begin{proposition}\label{prop:jensen} (Value of adaptation as a Jensen gap.)
For a random sensitivity profile $\bsigma$ with mean $\bar{\bsigma} := \E\,\bsigma$, define the oracle value function
\begin{equation}
  V^*(\bsigma) := \max_{\substack{\sum_t K_t = B\\ K_t \ge 1}}\;\sum_t \sigma_t\,a_{K_t}.
  \label{eq:static-oracle}
\end{equation}
Then the following statements hold. (i) $V^*$ is convex. (ii) The best fixed allocation attains expected surrogate value $V^*(\bar{\bsigma})$, and the expected shortfall of the best fixed allocation relative to the instance-wise oracle equals the Jensen gap
\begin{equation}
  J(\bsigma) \;:=\; \E\bigl[V^*(\bsigma)\bigr] - V^*(\bar{\bsigma}) \;\ge\; 0.
  \label{eq:jensen-gap}
\end{equation}
\end{proposition}

The proof of Proposition~\ref{prop:jensen} is provided in Appendix~\ref{app:proofs}. Proposition~\ref{prop:jensen} identifies $J(\bsigma)$ as the population value of clairvoyance relative to the best static schedule: it is exactly the expected improvement from re-optimizing the allocation after the instance-level sensitivity profile is revealed. The quantity $J(\bsigma)$ is zero when sensitivity profiles are homogeneous, and it is strictly positive whenever profile variation changes the optimal allocation. Combining Theorems~\ref{thm:offline} and~\ref{thm:lower} with Proposition~\ref{prop:jensen} yields the design principle that guides the algorithm of Section~\ref{sec:algorithm}: anchor the allocation \emph{offline} at the population profile $\bar{\bsigma}$, where it enjoys the optimality guarantees of Theorem~\ref{thm:offline}, and \emph{adjust online} when observed instance-level signals reveal recoverable Jensen gap. Theorem~\ref{thm:lower} quantifies the adversarial regime in which this online value is unavailable.

\subsection{Extension to General Local Search Operators}
\label{sec:general-operators}

The results so far concern random search. Practical local operators reuse information across candidates, perturb incumbents, or interleave exploration with exploitation. This subsection shows that the fixed-profile surrogate allocation structure extends under four regularity conditions on the operator. The only operator-specific object the extension needs is a per-step gain curve $\phi_K^{\cL}$, the expected improvement from $K$ within-step iterations under the linearized score, defined precisely below. For score-dependent operators, we use a \emph{projected-score model} in which accept-or-reject comparisons are resolved by the linearized score; Appendix~\ref{app:general} quantifies the window on which this model and exact comparisons can differ. Random search retains the best of $K$ independent draws and is covered exactly by Theorem~\ref{thm:loc-scale}(iii). Operators that update an incumbent through sequential comparisons, such as $\epsilon$-greedy and local perturbation, use the projected-score model because the second-order remainder can propagate through the incumbent to subsequent iterations.

A local search operator $\cL$ at step $t$ produces candidate noise iterates $\xi^{(1)}, \ldots, \xi^{(K)}$ in $\R^d$, each evaluated through the transition; iterative operators may let $\xi^{(j+1)}$ depend on the earlier iterates. The gain from $K$ iterations is
\begin{equation}
  G_t^{\cL}(K) := \E\Bigl[\max_{1 \le j \le K} \Psi_t\bigl(\bar{X}_{t-h} + g_t\sqrt{h}\,\xi^{(j)};\,c\bigr)\,\Big|\,X_t,c\Bigr] - \mu_t,
  \label{eq:gain-general}
\end{equation}
with the convention $G_t^{\cL}(0) := 0$. Define the projected gain sequence $\phi_K^{\cL} := \E[\max_{1 \le j \le K} \langle u, \xi^{(j)}\rangle]$, where $\{\xi^{(j)}\}$ are the iterates generated when the linearized score direction is the unit vector $u$, and set $\phi_0^{\cL} := 0$.

\begin{assumption}\label{asm:local-search} (Local search regularity.)
The operator $\cL$ satisfies the following four conditions.
\begin{enumerate}
\item[(A1)] \emph{Strict monotone improvement.} $\phi_K^{\cL}$ is strictly increasing in $K$ for $K \ge 1$.
\item[(A2)] \emph{Diminishing returns.} The marginal gains $\Delta\phi_K^{\cL} := \phi_{K+1}^{\cL} - \phi_K^{\cL}$ are non-increasing for $K \ge 1$.
\item[(A3)] \emph{Direction independence.} The projected gain $\E[\max_{1 \le j \le K}\langle u, \xi^{(j)}\rangle]$ does not depend on the unit direction $u$.
\item[(A4)] \emph{Rotational equivariance.} For any orthogonal matrix $Q$, the iterates generated under score direction $Qu$ are distributed as $\{Q\xi^{(j)}\}$, where $\{\xi^{(j)}\}$ are the iterates generated under direction $u$.
\end{enumerate}
\end{assumption}

In words, conditions (A1) and (A2) say that additional within-step iterations keep helping but with diminishing returns, while (A3) and (A4) say that the operator does not privilege any particular direction of the noise space. The direction symmetry holds whenever the operator is built from isotropic primitives, such as Gaussian draws and uniformly directed perturbations, and accepts or rejects candidates on the basis of scores alone; condition (A3) is the consequence of (A4) that the proofs invoke directly, and we list it separately for clarity. Under these conditions, the factorization of Theorem~\ref{thm:loc-scale}(iii) generalizes as follows.

\begin{theorem}\label{thm:general-gain} (General gain factorization.)
Suppose Assumption~\ref{asm:smooth} holds, conditions (A3) and (A4) of Assumption~\ref{asm:local-search} hold for $\cL$, and the iterate moment bound $\E[\max_{1\le j\le K}\|\xi^{(j)}\|^2] \le M_K(\cL, d) < \infty$ holds for some sequence $M_K$. Then
\begin{equation}
  G_t^{\cL}(K) = \sigma_t\,\phi_K^{\cL} + O_{K,d,L_t,\cL}(g_t^2 h),
  \label{eq:general-gain-factorization}
\end{equation}
where $\sigma_t$ is the sensitivity \eqref{eq:sensitivity} and the implicit constant depends on $K$, $d$, $L_t$, and $M_K(\cL, d)$, but not on $g_t$ or $\nabla\Psi_t$. If in addition conditions (A1) and (A2) hold, the allocation consequences of Corollary~\ref{cor:general-alloc} apply.
\end{theorem}

The proof of Theorem~\ref{thm:general-gain} is provided in Appendix~\ref{app:general}, together with a lemma verifying the moment bound, uniformly in $(c, x_t)$, for each operator considered in this paper; the proof linearizes every iterate via Lemma~\ref{lem:taylor}, chooses an arbitrary unit direction in the degenerate case $\nabla\Psi_t(\bar X_{t-h};c)=0$, and uses rotational equivariance only to show that the projected gain sequence is direction-free. The theorem establishes what we call the \emph{sensitivity-gain separation}: the step-dependent factor $\sigma_t$ and the operator-dependent sequence $\phi_K^{\cL}$ decouple multiplicatively. Given a fixed operator-specific gain sequence $\phi_K^{\cL}$, the fixed-profile allocation results stated in Corollary~\ref{cor:general-alloc} follow by replacing $a_K$ with $\phi_K^{\cL}$; the resulting schedule still depends jointly on the sensitivity profile and that gain sequence.

\begin{corollary}\label{cor:general-alloc} (Allocation for general operators.)
Under the conditions of Theorem~\ref{thm:general-gain}, with conditions (A1) and (A2) of Assumption~\ref{asm:local-search} in force and instance-independent sensitivities, the integer allocation problem $\max\{\sum_t \sigma_t \phi^{\cL}_{K_t} : \sum_t K_t = B,\ K_t\in\mathbb Z_{\ge1}\}$ has the following counterparts of Theorem~\ref{thm:offline}: the optimizer is instance-independent, satisfies the common marginal-threshold conditions \eqref{eq:marginal-threshold} with $\Delta a$ replaced by $\Delta\phi^{\cL}$, and is monotone in sensitivity. When $\phi_K^{\cL}$ is strictly concave, decreasing marginals yield the same water-filling characterization; when $\phi_K^{\cL}$ is affine, every optimizer concentrates the entire discretionary budget on the most sensitive steps. Furthermore, the conclusions of Proposition~\ref{prop:jensen} hold with $V^*$ defined through $\phi^{\cL}$.
\end{corollary}

The proof of Corollary~\ref{cor:general-alloc} is provided in Appendix~\ref{app:general}. The dichotomy in the corollary shows that the choice of local operator influences not only the rate of per-step improvement but also the \emph{shape} of the optimal schedule. Appendix~\ref{app:general} instantiates the framework on three operators and verifies Assumption~\ref{asm:local-search} for each; Table~\ref{tab:phi-summary} summarizes the outcome. Random search yields the strictly concave sequence $\phi_K = a_K$ and therefore water-filling schedules. Local perturbation search, which repeatedly perturbs the incumbent by a uniformly directed step of radius scale $\lambda$, yields in the projected-score idealization the exactly affine sequence $\phi_K = c_d(\lambda)(K-1)$ with drift constant $c_d(\lambda) = \lambda\sqrt d\,\Gamma(d/2)/(4\sqrt\pi\,\Gamma((d+1)/2))$, and therefore concentrated schedules. The $\epsilon$-greedy operator interpolates between the two regimes.

\begin{table}[t]
\centering
\caption{Operator-specific gain sequences in the projected-score idealization. All three operators satisfy Assumption~\ref{asm:local-search}; derivations and verifications appear in Appendix~\ref{app:general}.}
\label{tab:phi-summary}
\small
\begin{tabular}{@{}llll@{}}
\toprule
Operator $\cL$ & Gain sequence $\phi_K^{\cL}$ & Growth rate & Status \\
\midrule
Random search & $a_K \sim \sqrt{2\log K}$ & $O(\sqrt{\log K})$ & Exact \\
$\epsilon$-greedy & $\ge \E[a_{N+1}]$, $N\sim\mathrm{Bin}(K{-}1,\epsilon)$ & $\Omega(\sqrt{\log K})$ & Lower bound \\
Local perturbation & $c_d(\lambda)\,(K-1)$ & $O(K)$ & Exact in the idealization \\
\bottomrule
\end{tabular}
\end{table}

\section{The \GAINS{} Algorithm}
\label{sec:algorithm}

This section converts the design principle derived in Section~\ref{sec:adaptivity} into an implementable procedure. \GAINS{} (Global Adaptive Inference-time Noise Scheduling) operates at two granularities. An \emph{offline profiling} stage, executed once per model and data distribution, estimates the population sensitivity pattern and anchors a coarse allocation guided by Theorem~\ref{thm:offline}. An \emph{online control} stage then adjusts the realized expenditure on each individual instance, seeking to recover part of the Jensen gap identified in Proposition~\ref{prop:jensen} while never exceeding the overall budget.

\subsection{Offline Profiling}
\label{sec:offline}

The offline stage decides, before any instance arrives, how many local-search iterations to plan at each step. The difficulty is that, when deciding at step $t$, the scheduler knows neither the additional improvement that further search at $t$ would bring nor the improvement it forgoes by conserving budget for later steps. Offline profiling addresses both unknowns jointly by measuring how responsive the verifier score is to local search at each step. Specifically, we execute a one-time profiling pass on a small calibration set: noise search is run with a uniform iteration count at every step, and the verifier score is recorded after each iteration. From these records we compute, for each step $t$, the average per-iteration score gain and its variability across calibration instances. Motivated by the sensitivity interpretation in Theorem~\ref{thm:loc-scale}, these statistics summarize the population pattern of where search is most valuable along the trajectory. The profiled pattern is strongly model-dependent; in our profiling runs, the gains of Stable Diffusion concentrate on early steps, whereas those of the EDM sampler cluster in a contiguous middle segment. We partition the steps into high- and low-value regions, allocate a larger share of $B$ to the former, and derive a planned per-step budget $\{K_t\}$ with $\sum_t K_t = B$. The resulting schedule is tailored to the model and data distribution but fixed across instances; it encodes \emph{where along the trajectory search matters most}.

\subsection{Online Control}
\label{sec:online}

The online stage refines, at deployment time, how the planned budget is spent on each individual instance. At step $t$, local search runs up to $K_t$ iterations, but the controller may stop early or, within a small slack window, shift expenditure across steps. We distinguish the raw score $\hat s_t^{(j)}$ of the candidate produced at iteration $j$ from the running best $s_t^{(j)} := \max(s_t^{(j-1)}, \hat s_t^{(j)})$, initialized at $s_t^{(0)} := \hat s_t^{(1)}$. The controller monitors two statistics: the incremental gain $g_t^{(j)} := s_t^{(j)} - s_t^{(j-1)} \ge 0$, which vanishes exactly when iteration $j$ fails to improve the incumbent, and the empirical variance $\mathrm{Var}_{i \le j}(\hat s_t^{(i)})$ of the candidate scores observed at the current step.

The two statistics answer complementary questions, and their conjunction determines when to stop. The incremental gain asks whether the search at the current step is close to exhaustion: when recent gains approach zero, additional candidates purchase little. The empirical variance asks whether the current step is \emph{important}: by Theorem~\ref{thm:loc-scale}(ii), within-step variance tracks $\sigma_t^2$ to leading order for independent candidates and motivates its use as a practical sensitivity indicator. Stopping occurs when \emph{both} statistics are small. A small gain can mark a high-sensitivity step awaiting a fortunate draw, while a small variance can mark a low-sensitivity step whose iterations continue to nudge the running best upward; the conjunction distinguishes these regimes. Figure~\ref{fig:stopping} illustrates the decision geometry: the contour $\sigma_t \cdot \Delta a_K = \lambda^*$ represents the common marginal-threshold condition \eqref{eq:marginal-threshold}, and the rectangular two-threshold rule operationalizes this principle online. Because the running best is monotone by construction, a non-improving iteration leaves it unchanged. The budget accounting guarantees that the realized total never exceeds $B$; expenditure saved by early stopping in high-value regions is offset against later steps within the slack window, and in all of our experimental runs the realized total equaled $B$. Algorithm~\ref{alg:gains} states the procedure.

\begin{figure}[t]
\centering
\includegraphics[width=0.44\linewidth]{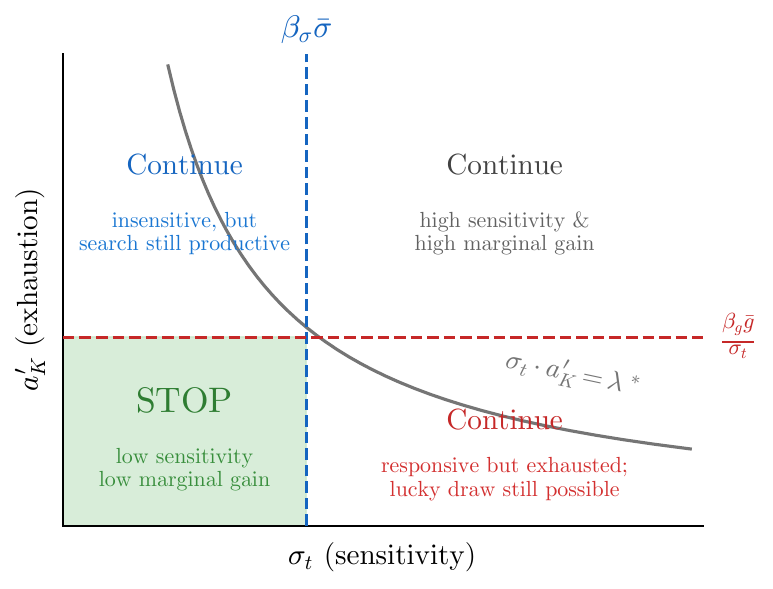}
\caption{The online stopping decision in the $(\sigma_t, \Delta a_K)$ plane. The horizontal threshold detects small recent gains; the vertical threshold detects low sensitivity through small empirical variance. Stopping is triggered only when both criteria fire (shaded region). The contour $\sigma_t\cdot\Delta a_K=\lambda^*$ depicts the common marginal threshold from Theorem~\ref{thm:offline}; the rectangular rule operationalizes the corresponding two-threshold decision in \GAINS{}.}
\label{fig:stopping}
\end{figure}

\begin{algorithm}[t]
\caption{\GAINS{}: Global Adaptive Inference-time Noise Scheduling}
\label{alg:gains}
\begin{algorithmic}[1]
\Require Steps $\{t_T,\dots,t_1\}$; total budget $B$; offline allocation $\{K_t\}$ with $\sum_t K_t = B$; slack $\delta$; gain window $W_g$; threshold coefficients $(\beta_g,\beta_\sigma)$; black-box routine $\textsc{LocalIter}(t)$ returning one raw candidate score $\hat s$ and the within-step empirical standard deviation $\sigma$.
\Ensure Realized per-step iteration counts $\{\widehat K_t\}$.
\State $R \gets B$;\quad $\cH_g \gets \emptyset$;\quad $\cH_\sigma \gets \emptyset$ \Comment{remaining budget; cross-step histories}
\For{$t = T, T-1, \dots, 1$}
    \State $K \gets \min(K_t, R)$;\quad $K_{\max}\gets \min(K+\delta,\, R)$;\quad $j_{\mathrm{watch}}\gets \max(2,\, K-\delta)$;\quad $\cH\gets\emptyset$
    \State $(\hat s^{(1)}, \sigma^{(1)}) \gets \textsc{LocalIter}(t)$;\quad $s^{(1)} \gets \hat s^{(1)}$;\quad $\widehat K_t \gets 1$;\quad $R \gets R-1$
    \For{$j = 2,\dots,K_{\max}$}
        \State $(\hat s^{(j)}, \sigma^{(j)}) \gets \textsc{LocalIter}(t)$
        \State $s^{(j)} \gets \max(s^{(j-1)}, \hat s^{(j)})$ \Comment{running best; a non-improving iteration leaves it unchanged}
        \State push $g^{(j)}\gets s^{(j)}-s^{(j-1)}$ into $\cH$ (retain last $W_g$);\quad $\widetilde g^{(j)}\gets\mathrm{Mean}(\cH)$;\quad $\widehat K_t\gets\widehat K_t+1$;\quad $R \gets R-1$
        \If{$R=0$} \textbf{break} \EndIf
        \If{$j \ge j_{\mathrm{watch}}$ \textbf{and} $|\cH_g|>0$ \textbf{and} $|\cH_\sigma|>0$ \textbf{and} $\widetilde g^{(j)} < \beta_g\,\mathrm{Mean}(\cH_g)$ \textbf{and} $\sigma^{(j)} < \beta_\sigma\,\mathrm{Mean}(\cH_\sigma)$}
            \State \textbf{break} \Comment{two-threshold stopping rule}
        \EndIf
    \EndFor
    \State append $\mathrm{Mean}(\cH)$ to $\cH_g$ if $\cH\neq\emptyset$;\quad append $\sigma^{(\widehat K_t)}$ to $\cH_\sigma$
\EndFor
\State \Return $\{\widehat K_t\}$
\end{algorithmic}
\end{algorithm}

Taken together, the two stages translate the structural results into an implementable design. Offline profiling operationalizes the fixed-profile allocation principle of Theorem~\ref{thm:offline}, while online control implements a lightweight version of its marginal-threshold logic. The lower bound of Theorem~\ref{thm:lower} motivates anchoring the schedule offline, and Proposition~\ref{prop:jensen} identifies the instance-specific value that online adjustment can recover. The numerical experiments of the next section evaluate the resulting procedure.

\section{Numerical Experiments}
\label{sec:experiments}

This section evaluates \GAINS{} numerically. The experiments are organized around the central prediction of the theory: under a fixed local-search budget, reallocating budget toward more sensitive steps should improve upon the uniform benchmark when sensitivity heterogeneity is material, and the improvement should widen as the sensitivity profile becomes more dispersed, locally as in Theorem~\ref{thm:offline}(iv) and, at fixed total sensitivity, globally in the majorization sense of Proposition~\ref{prop:majorization}. We examine this prediction across model families (Section~\ref{sec:exp-scaling}), decompose the contribution of the offline and online stages (Section~\ref{sec:exp-ablation}), assess out-of-sample generalization of the offline schedule (Section~\ref{sec:exp-prompt}), verify robustness to the choice of local operator (Section~\ref{sec:exp-operator}), and study a low-dispersion regime in which the theory anticipates smaller gains from reallocation (Section~\ref{sec:exp-flow}).

\subsection{Experimental Setup}
\label{sec:exp-setup}

We evaluate \GAINS{} on three families of samplers: Stable Diffusion \citep[SD;][]{rombach2022high} for text-conditional generation, EDM \citep{karras2022elucidating} for class-conditional generation, and a flow-based sampler. Sample quality is measured by two lightweight verifiers, \textbf{Brightness} (average luminance) and \textbf{Compressibility} (compressed file size), both deterministic functions of the generated image. The benchmark policy, denoted \emph{Uniform}, runs the same $\epsilon$-greedy local search \citep{ramesh2025nts}, which alternates random exploration with local perturbations of the incumbent, under an even allocation across steps; \GAINS{} replaces the even allocation with offline profiling and online control, redistributing saved evaluations so that the realized total matches the budget $B$ exactly in all reported runs. Both policies therefore expend identical budgets, and all comparisons are at matched cost. We use Uniform as the benchmark because it is the standard fixed-profile policy in this literature, and report budget savings at matched quality as the primary measure of economic efficiency. All runs use a single NVIDIA A100 graphics processor; wall-clock time is dominated by forward passes through the denoising network and consequently scales with the budget.

Hyperparameters are as follows. For EDM ($T=18$): exploration probability $\epsilon = 0.4$, perturbation radius $\lambda = 0.15$, candidates per search iteration $N = 4$, threshold coefficients $\beta_g = 0.3$ and $\beta_\sigma = 0.7$, slack $\delta = 2$. For SD ($T=50$): $N = 4$, $\beta_g = 0.1$, $\beta_\sigma = 0.8$. All runs use the running-best rule with gain window $W_g = 4$ and planned allocations $\{K_t\}$ from offline profiling. Each configuration is averaged over 20 independent replications, and reported uncertainties are standard deviations across replications; the generalization study uses 20 prompts with 10 replications each.

In this section, the reported budget $B$ is the total number of local-search iterations performed across all denoising steps. Each iteration evaluates $N$ candidates, so the corresponding network cost is $N B$ function evaluations (NFE); for the reported settings, $N=4$.

\subsection{Budget Sensitivity Across Model Families}
\label{sec:exp-scaling}

\begin{table}[t]
\centering
\caption{Verifier scores under matched local-search budgets. Each entry is the mean over 20 replications, plus or minus one standard deviation. Boldface marks the better policy at each budget.}
\label{tab:budgets}
\small
\begin{tabular}{llccc}
\toprule
& & \multicolumn{3}{c}{\textbf{Stable Diffusion} (budget $B$)} \\
\cmidrule(lr){3-5}
\textbf{Verifier} & \textbf{Policy} & 400 & 500 & 800 \\
\midrule
\multirow{2}{*}{Brightness} & Uniform & $0.7025\pm0.047$ & $0.7094\pm0.063$ & $0.7511\pm0.062$ \\
 & \GAINS{} & $\mathbf{0.7248\pm0.056}$ & $\mathbf{0.7346\pm0.060}$ & $\mathbf{0.7832\pm0.046}$ \\
\midrule
\multirow{2}{*}{Compressibility} & Uniform & $0.8833\pm0.017$ & $0.8825\pm0.017$ & $0.8892\pm0.020$ \\
 & \GAINS{} & $\mathbf{0.8946\pm0.016}$ & $\mathbf{0.9004\pm0.015}$ & $\mathbf{0.9008\pm0.012}$ \\
\midrule
& & \multicolumn{3}{c}{\textbf{EDM} (budget $B$)} \\
\cmidrule(lr){3-5}
\textbf{Verifier} & \textbf{Policy} & 144 & 180 & 288 \\
\midrule
\multirow{2}{*}{Brightness} & Uniform & $0.9507\pm0.015$ & $0.9617\pm0.016$ & $0.9787\pm0.018$ \\
 & \GAINS{} & $\mathbf{0.9887\pm0.005}$ & $\mathbf{0.9904\pm0.004}$ & $\mathbf{0.9988\pm0.001}$ \\
\midrule
\multirow{2}{*}{Compressibility} & Uniform & $0.6714\pm0.017$ & $0.6774\pm0.012$ & $0.6901\pm0.012$ \\
 & \GAINS{} & $\mathbf{0.6845\pm0.007}$ & $\mathbf{0.7003\pm0.009}$ & $\mathbf{0.7165\pm0.009}$ \\
\bottomrule
\end{tabular}
\end{table}

Table~\ref{tab:budgets} reports verifier scores across budgets for SD and EDM. Two observations stand out. First, \GAINS{} improves upon the uniform benchmark at every budget, for every verifier, and for both model families; on SD the Brightness margin additionally widens with the budget, from $+0.0223$ at $B = 400$ to $+0.0321$ at $B = 800$. Second, and of primary economic interest, the improvements translate into substantial budget savings at matched quality. On SD, \GAINS{} at $B=400$ exceeds the Uniform policy at $B=500$ on Brightness, a saving of 20 percent, and matches the Uniform policy at $B=800$ on Compressibility, a saving of 50 percent. On EDM, \GAINS{} at $B=144$ exceeds Uniform at $B=288$ on Brightness (a 50 percent saving), and \GAINS{} at $B=180$ exceeds Uniform at $B=288$ on Compressibility (a 37.5 percent saving). The larger margins observed on EDM are consistent with the dispersion comparative statics of Proposition~\ref{prop:majorization}: the profiled EDM sensitivities concentrate on a middle segment of the trajectory and are more dispersed than the early-spread profile of SD. We also observe that replication-to-replication variability frequently tightens under \GAINS{}, most markedly on EDM, for example $\pm 0.005$ against $\pm 0.015$ at $B = 144$, indicating that on such configurations principled reallocation not only raises the mean but also stabilizes the outcome distribution.

\subsection{Decomposition of the Offline and Online Stages}
\label{sec:exp-ablation}

\begin{table}[t]
\centering
\caption{Decomposition of the two stages on Stable Diffusion at fixed budget $B = 400$. \emph{Offline only} applies the profiled allocation without online control.}
\label{tab:ablation}
\small
\begin{tabular}{lccc}
\toprule
\textbf{Verifier} & \textbf{Uniform} & \textbf{Offline only} & \textbf{\GAINS{} (offline + online)} \\
\midrule
Brightness & $0.7025\pm0.047$ & $0.7158\pm0.065$ & $\mathbf{0.7248\pm0.056}$ \\
Compressibility & $0.8833\pm0.017$ & $0.8873\pm0.018$ & $\mathbf{0.8946\pm0.016}$ \\
\bottomrule
\end{tabular}
\end{table}

Table~\ref{tab:ablation} isolates the contribution of each stage at a fixed local-search budget of $B=400$ iterations on SD. Replacing the uniform allocation with the offline schedule alone lifts Brightness by $+0.0133$ and Compressibility by $+0.0040$, showing that coarse population-level budgeting already captures useful structure, consistent with the instance-independence property of Theorem~\ref{thm:offline}(i) when sensitivities are approximately common across instances. Layering the online controller on top yields a further $+0.0090$ and $+0.0073$, respectively. This second increment is consistent with the Jensen-gap logic of Proposition~\ref{prop:jensen}: because instance-level sensitivity profiles vary around the population mean, an allocation tuned to the mean leaves recoverable value, and the online controller appears to recover a portion of it. The decomposition thus matches the two-part design principle of Section~\ref{sec:adaptivity} component by component.

\subsection{Out-of-Sample Generalization of the Offline Schedule}
\label{sec:exp-prompt}

\begin{table}[t]
\centering
\caption{Generalization to a broader prompt distribution: 20 prompts, 10 replications each, on Stable Diffusion.}
\label{tab:prompts}
\small
\begin{tabular}{lcc}
\toprule
\textbf{Policy} & \textbf{Brightness} & \textbf{Compressibility} \\
\midrule
Uniform & $0.6949\pm0.076$ & $0.7938\pm0.089$ \\
\GAINS{} & $\mathbf{0.7213\pm0.075}$ & $\mathbf{0.8076\pm0.096}$ \\
\bottomrule
\end{tabular}
\end{table}

We evaluate the offline schedule on 20 prompts with 10 independent replications each to measure transfer beyond the calibration set. Table~\ref{tab:prompts} shows that \GAINS{} raises mean Brightness from $0.6949$ to $0.7213$ and mean Compressibility from $0.7938$ to $0.8076$, while dispersion across replications remains comparable between the two policies. The gains persist across this diverse prompt set, indicating that the profiled schedule captures structure at the model level in line with the instance-independence property of Theorem~\ref{thm:offline}(i).

\subsection{Robustness to the Local Search Operator}
\label{sec:exp-operator}

\begin{table}[t]
\centering
\caption{Global scheduling combined with two qualitatively different local operators on Stable Diffusion. \GAINS{} improves upon the uniform benchmark under both operators.}
\label{tab:operators}
\small
\begin{tabular}{lcccc}
\toprule
& \multicolumn{2}{c}{\textbf{Uniform}} & \multicolumn{2}{c}{\textbf{\GAINS{}}} \\
\cmidrule(lr){2-3} \cmidrule(lr){4-5}
\textbf{Verifier} & Zero-order & Random & Zero-order & Random \\
\midrule
Brightness & $0.4920\pm0.072$ & $0.7382\pm0.039$ & $\mathbf{0.5080\pm0.075}$ & $\mathbf{0.7457\pm0.045}$ \\
Compressibility & $0.5727\pm0.066$ & $0.8417\pm0.027$ & $\mathbf{0.5832\pm0.061}$ & $\mathbf{0.8454\pm0.025}$ \\
\bottomrule
\end{tabular}
\end{table}

We test whether global scheduling remains effective when the within-step search rule changes. Corollary~\ref{cor:general-alloc} shows that the same allocation structure extends across the admissible local operators it covers, with the optimal schedule determined jointly by the sensitivity profile and the operator-specific gain sequence. Table~\ref{tab:operators} pairs both policies with zero-order perturbation search, whose idealization is the local-perturbation operator of Appendix~\ref{app:general}, and random search. The operators differ in absolute performance because their gain sequences differ (Table~\ref{tab:phi-summary}): random search has the logarithmic sequence $a_K$, whereas pure perturbation search has an affine sequence that depends on the perturbation radius. Across both operators, \GAINS{} improves upon Uniform on both verifiers, by $+0.0160$ and $+0.0105$ under zero-order search and by $+0.0075$ and $+0.0037$ under random search, in line with the sensitivity-gain separation of Theorem~\ref{thm:general-gain}.

\subsection{A Low-Dispersion Regime: Flow-Based Samplers}
\label{sec:exp-flow}

\begin{table}[t]
\centering
\caption{Flow-based sampler with stochasticity injected at selected steps. Margins are smaller than in Table~\ref{tab:budgets}, consistent with the narrower sensitivity dispersion among the eligible stochastic steps.}
\label{tab:flow}
\small
\begin{tabular}{llccc}
\toprule
& & \multicolumn{3}{c}{Budget $B$} \\
\cmidrule(lr){3-5}
\textbf{Verifier} & \textbf{Policy} & 80 & 100 & 150 \\
\midrule
\multirow{2}{*}{Brightness} & Uniform & $0.5492\pm0.003$ & $0.5499\pm0.003$ & $0.5515\pm0.004$ \\
 & \GAINS{} & $\mathbf{0.5507\pm0.003}$ & $\mathbf{0.5519\pm0.003}$ & $\mathbf{0.5526\pm0.003}$ \\
\midrule
\multirow{2}{*}{Compressibility} & Uniform & $0.8135\pm0.005$ & $0.8156\pm0.004$ & $0.8157\pm0.006$ \\
 & \GAINS{} & $\mathbf{0.8162\pm0.003}$ & $\mathbf{0.8170\pm0.006}$ & $\mathbf{0.8193\pm0.005}$ \\
\bottomrule
\end{tabular}
\end{table}

Flow-based samplers advance the state deterministically. Following \citet{singh2024stochastic} and \citet{kim2025rbf}, we inject Gaussian stochasticity at a small set of selected steps, converting the deterministic solver into a stochastic one with matching marginal distributions, after which the formulation of Section~\ref{sec:framework} applies without modifying the pretrained network. Because only a few steps are eligible for search and their profiled sensitivities are comparatively homogeneous, this setting isolates the low-dispersion regime in Theorem~\ref{thm:offline}(iv). Table~\ref{tab:flow} shows that \GAINS{} improves upon Uniform at every budget, with \GAINS{} at $B=80$ exceeding Uniform at $B=100$ on Brightness (a 20 percent saving) and \GAINS{} at $B=100$ exceeding Uniform at $B=150$ on Compressibility (a 33 percent saving). The absolute margins range from $+0.001$ to $+0.004$, smaller than those in Table~\ref{tab:budgets}, as predicted for a low-dispersion profile. As in Section~\ref{sec:exp-scaling}, \GAINS{} also reduces run-to-run variability on several configurations, for example from $0.0058$ to $0.0046$ for Compressibility at $B = 150$.

\section{Concluding Remarks}
\label{sec:conclusion}

We study where to place inference-time computation along a diffusion denoising trajectory. We formulate this question as a computational budget allocation problem and decompose it into local search operators coupled to a global scheduler. For the scheduler we then developed a structural theory with four conclusions. First, the expected per-step gain factorizes into an endogenous sensitivity parameter and a universal sample-size factor (Theorem~\ref{thm:loc-scale}). Second, for a fixed sensitivity profile, this reduction yields a standard separable concave integer allocation problem with common marginal-threshold conditions of water-filling type (Theorem~\ref{thm:offline}); at fixed total sensitivity, greater dispersion increases the surplus over uniform allocation (Proposition~\ref{prop:majorization}). Third, no adaptive policy escapes worst-case regret linear in the trajectory length (Theorem~\ref{thm:lower} and Proposition~\ref{prop:transfer}, unconditionally within a stylized instance class by Proposition~\ref{prop:realizable}). Fourth, the value of instance-adaptive allocation on non-adversarial populations equals a nonnegative Jensen gap (Proposition~\ref{prop:jensen}). For a general family of local search operators, the operator-specific gain sequence determines the marginal-allocation structure: strictly concave gains yield decreasing-marginal water-filling conditions, whereas affine gains concentrate discretionary budget on the most sensitive steps (Theorem~\ref{thm:general-gain} and Corollary~\ref{cor:general-alloc}). The resulting algorithm attains the quality of the uniform benchmark with 20 to 50 percent fewer function evaluations across three families of samplers.

For deployment at scale, computational budgets can be anchored offline using population-level sensitivity profiles estimated from a small calibration run. Online adaptation then recovers instance-specific Jensen-gap value while retaining the offline worst-case guarantee quantified by Theorem~\ref{thm:lower}. The sensitivity parameter itself, the product of the schedule's noise injection and the verifier's gradient response, provides an interpretable diagnostic for where in a sequential generation pipeline additional computation purchases the most quality.

The analysis has two limitations. For score-dependent local operators, the guarantees describe projected-score comparisons; extending them to exact accept-or-reject decisions requires quantifying the effect of the comparison window in Appendix~\ref{app:general}. Across stages, the analysis linearizes how search effort changes later opportunities. Future work can model this dependence directly and study how to allocate computation among verifiers with different costs and when to revisit earlier steps (Appendix~\ref{app:mdp}).


\appendix

\section{Proofs of the Results in Sections \ref{sec:loc-scale} to \ref{sec:adaptivity}}
\label{app:proofs}

This appendix proves the results of Sections~\ref{sec:loc-scale}--\ref{sec:adaptivity}, in the order in which they are stated. We begin with an auxiliary expansion of the verifier score, the source of the linear-in-noise leading term that drives the location-scale structure.

\begin{lemma}\label{lem:taylor} (Score Taylor expansion.)
Suppose Assumption~\ref{asm:smooth} holds. Then, conditional on $(X_t, c)$,
\begin{equation}
  S_t \;=\; \Psi_t(\bar{X}_{t-h};\, c)
        \;+\; g_t\sqrt{h}\,\bigl\langle \nabla\Psi_t(\bar{X}_{t-h};\, c),\, \xi \bigr\rangle
        \;+\; R_t,
  \label{eq:taylor-score}
\end{equation}
where the remainder satisfies $|R_t| \le \tfrac{1}{2}\,L_t\,g_t^2\,h\,\|\xi\|^2$ almost surely, and hence $\E[\,|R_t|\mid X_t, c\,] \le \tfrac{1}{2}\,L_t\,g_t^2\,h\,d$.
\end{lemma}

\proof{Proof of Lemma~\ref{lem:taylor}.}
By \eqref{eq:em-step}, $X_{t-h} = \bar X_{t-h} + g_t\sqrt h\,\xi$. A second-order Taylor expansion of $\Psi_t(\cdot\,;c)\in C^2$ about $\bar X_{t-h}$ gives, for some $\widetilde X$ on the segment $[\bar X_{t-h}, X_{t-h}]$,
\[
  S_t = \Psi_t(\bar X_{t-h};c) + g_t\sqrt h\,\bigl\langle \nabla\Psi_t(\bar X_{t-h};c),\, \xi\bigr\rangle + R_t,
  \qquad
  R_t = \tfrac12 g_t^2 h\,\xi^\top \nabla^2\Psi_t(\widetilde X;c)\,\xi,
\]
which is \eqref{eq:taylor-score}. By Assumption~\ref{asm:smooth},
\[
  |R_t| \le \tfrac12\,\|\nabla^2\Psi_t(\widetilde X;c)\|_{\mathrm{op}}\,g_t^2 h\,\|\xi\|^2 \le \tfrac12 L_t g_t^2 h\,\|\xi\|^2,
\]
and $\E[\cdot\mid X_t,c]$ with $\E[\|\xi\|^2\mid X_t,c]=d$ gives $\E[|R_t|\mid X_t,c]\le\tfrac12 L_t g_t^2 h\,d$. \Halmos
\endproof

\proof{Proof of Theorem~\ref{thm:loc-scale}.}
\emph{Part (i).} Suppose first that $\|\nabla\Psi_t\|>0$, and set
\[
  u_t:=\frac{\nabla\Psi_t}{\|\nabla\Psi_t\|},
  \qquad
  Z_t:=\langle u_t,\xi\rangle.
\]
Conditional on $(X_t,c)$, the unit vector $u_t$ is deterministic and $\xi\sim\mathcal N(0,I_d)$, so $Z_t\sim\mathcal N(0,1)$ and $\langle\nabla\Psi_t,\xi\rangle=\|\nabla\Psi_t\|Z_t$. Lemma~\ref{lem:taylor} therefore gives $S_t=\mu_t+\sigma_tZ_t+R_t$, with $\sigma_t=g_t\sqrt h\,\|\nabla\Psi_t\|$. If $\nabla\Psi_t=0$, then $\sigma_t=0$ and the linear term vanishes, so the representation holds with any standard normal variable $Z_t$.

\emph{Part (ii).} Since $\xi\sim\mathcal N(0,I_d)$ gives $\E\|\xi\|^4=d(d+2)$, Lemma~\ref{lem:taylor} yields $\E[R_t^2\mid X_t,c]\le\tfrac14 L_t^2 g_t^4 h^2 d(d+2)$. Hence, by Cauchy--Schwarz,
\begin{align}
  \operatorname{Var}(S_t\mid X_t,c) &= \sigma_t^2 + 2\operatorname{Cov}(\sigma_t Z_t,R_t) + \operatorname{Var}(R_t), \nonumber\\
  |\operatorname{Cov}(\sigma_t Z_t,R_t)| &\le \sigma_t\sqrt{\E[R_t^2]} \le \tfrac12 L_t\sqrt{d(d+2)}\,\sigma_t g_t^2 h = O_{d,L_t,\|\nabla\Psi_t\|}(g_t^3 h^{3/2}), \nonumber\\
  \operatorname{Var}(R_t) &\le \E[R_t^2] = O_{d,L_t}(g_t^4 h^2). \nonumber
\end{align}
The last term is of lower order as $g_t\sqrt h\to0$, so $\operatorname{Var}(S_t\mid X_t,c)=\sigma_t^2+O_{d,L_t,\|\nabla\Psi_t\|}(g_t^3 h^{3/2})$.

\emph{Part (iii).} By Part (i), $S_t^{(k)}=\mu_t+\sigma_t Z_t^{(k)}+R_t^{(k)}$ with $Z_t^{(k)}\overset{\mathrm{iid}}{\sim}\mathcal N(0,1)$. Using $|\max_k(a_k+b_k)-\max_k a_k|\le\max_k|b_k|$ and $\sigma_t\ge0$,
\[
  \Bigl|\max_k S_t^{(k)} - \bigl(\mu_t+\sigma_t\max_k Z_t^{(k)}\bigr)\Bigr| \le \max_k|R_t^{(k)}| \le \sum_k |R_t^{(k)}|,
  \qquad
  \E\Bigl[\textstyle\sum_k |R_t^{(k)}|\,\Big|\,X_t,c\Bigr] \le \tfrac12 K L_t g_t^2 h\,d.
\]
Taking $\E[\cdot\mid X_t,c]$ and subtracting $\mu_t$,
\[
  G_t(K) = \sigma_t\,\E\bigl[\max_k Z_t^{(k)}\bigr] + O_{K,d,L_t}(g_t^2 h) = \sigma_t a_K + O_{K,d,L_t}(g_t^2 h),
\]
the implicit constant depending on $K,d,L_t$ but not on $g_t$ or $\nabla\Psi_t$. \Halmos
\endproof

The factorization aggregates across the trajectory by summing the stepwise gains, which we record for reference.

\begin{corollary}\label{cor:aggregate} (Aggregation along the trajectory.)
Under Assumption~\ref{asm:smooth}, for any fixed allocation $\{K_t\}_{t=1}^T$, conditional on the trajectory $\{X_t\}_{t=1}^T$,
\begin{equation}
  \sum_{t=1}^{T} G_t(K_t) = \sum_{t=1}^{T} g_t\sqrt{h}\,\|\nabla\Psi_t(\bar{X}_{t-h}; c)\|\,a_{K_t} + \sum_{t=1}^{T} O_{K_t,d,L_t}(g_t^2 h).
  \label{eq:aggregate-gain}
\end{equation}
Dropping the second-order remainder reduces the offline allocation problem to the concave resource allocation problem \eqref{eq:alloc} with per-step weights $\sigma_t = g_t\sqrt{h}\,\|\nabla\Psi_t\|$; Remark~\ref{rmk:regime} states the regime in which the omission is justified.
\end{corollary}

\proof{Proof of Theorem~\ref{thm:offline}.}
Set \(m_{t,k}:=\sigma_t\Delta a_k\). The Gaussian order-statistic marginals satisfy \(\Delta a_k>0\) and are strictly decreasing in \(k\) \citep[Ch.~4]{david2003order}; hence each sequence \(\{m_{t,k}\}_{k\ge1}\) is nonnegative and non-increasing (strictly positive and decreasing whenever \(\sigma_t>0\)), and
\[
  \sigma_t a_{K_t}=\sigma_t a_1+\sum_{k=1}^{K_t-1}m_{t,k}.
\]
Hence moving one unit from an active step \(s\) to a distinct step \(t\ne s\) changes the objective by \(m_{t,K_t}-m_{s,K_s-1}\). This one-unit exchange proves all the integer statements.

\emph{Part (i).} If \(\sigma_t\) depends only on \(t\), both the surrogate objective and its feasible set are identical for every \((c,x_t)\), so the optimizer set is instance-independent.

\emph{Part (ii).} If \(B=T\), the only feasible allocation is \(K_t^*=1\), and any \(\lambda^*\ge\max_t m_{t,1}\) works. If \(B>T\), let \(\mathcal A:=\{s:K_s^*>1\}\). Optimality gives
\[
  m_{t,K_t^*}\le m_{s,K_s^*-1}
  \qquad(s\in\mathcal A,\ t\ne s);
\]
for \(t=s\), the same inequality follows from diminishing marginals. Thus
\[
  L^*:=\max_t m_{t,K_t^*}
  \le
  U^*:=\min_{s\in\mathcal A}m_{s,K_s^*-1}.
\]
Any \(\lambda^*\in[L^*,U^*]\) gives \eqref{eq:marginal-threshold}, and an inactive step satisfies \(m_{t,1}\le L^*\le\lambda^*\).

\emph{Part (iii).} If \(\sigma_{t_1}>\sigma_{t_2}\) but \(K_{t_1}^*<K_{t_2}^*\), transferring one unit from \(t_2\) to \(t_1\) gains
\[
  \sigma_{t_1}\Delta a_{K_{t_1}^*}
  -
  \sigma_{t_2}\Delta a_{K_{t_2}^*-1}>0,
\]
because \(K_{t_1}^*+1\le K_{t_2}^*\), a contradiction. If two active steps have \(K_{t_1}^*=K_{t_2}^*=K\ge2\), this transfer gains \(\sigma_{t_1}\Delta a_K-\sigma_{t_2}\Delta a_{K-1}\), which is positive under the stated gap condition; hence the tie is impossible. Without the gap, \((3,3)\) can be optimal for \(T=2\), \(B=6\), and \(\bsigma=(1.01,1)\); floor-pinned ties occur, for example, at \(T=3\), \(B=4\), and \(\bsigma=(10,1,0.5)\). For a relaxed optimizer \(\widetilde K^*\), if \(\widetilde K_{t_1}^*>1\) and \(\widetilde K_{t_2}^*=1\), strict ordering is immediate. If both coordinates are interior, their first-order conditions are \(\sigma_t\tilde a'(\widetilde K_t^*)=\widetilde\lambda^*\); positivity and strict decrease of \(\tilde a'\) then imply \(\widetilde K_{t_1}^*>\widetilde K_{t_2}^*\) whenever \(\sigma_{t_1}>\sigma_{t_2}\). Only two coordinates jointly pinned at the floor may remain tied.

\emph{Part (iv).} When \(T\) divides \(B\), the uniform integer allocation is feasible, so \eqref{eq:dominance} follows immediately from optimality. For the continuous claims, choose \(v_1>\Delta a_1\) and \(v_K:=\tfrac12(\Delta a_{K-1}+\Delta a_K)\) for \(K\ge2\). Since \(v_{K+1}<\Delta a_K<v_K\), set \(u=v_K\), \(\ell=v_{K+1}\), and \(m=\Delta a_K\) on \([K,K+1]\). If \(m\ge(u+\ell)/2\), take \(f_K(s)=u-(u-\ell)s^p\), choosing \(p\ge1\) so that \(m=u-(u-\ell)/(p+1)\); if \(m\le(u+\ell)/2\), take \(f_K(s)=\ell+(u-\ell)(1-s)^p\), choosing \(p\ge1\) so that \(m=\ell+(u-\ell)/(p+1)\). In either case, \(f_K\) is positive, Lipschitz, and strictly decreasing on \([0,1]\), has endpoints \(v_K,v_{K+1}\), and integrates to \(\Delta a_K\). Matching these functions at the integers and setting
\[
  \tilde a'(K+s):=f_K(s),
  \qquad
  \tilde a(x):=a_1+\int_1^x\tilde a'(u)\,du
\]
gives a \(C^1\), strictly concave interpolant with positive, locally Lipschitz derivative. For \(B>T\), let \(\bar K:=B/T\). If \(\sigma_{t_1}>\sigma_{t_2}\), the feasible perturbation \(K_{t_1}=\bar K+\epsilon\), \(K_{t_2}=\bar K-\epsilon\) changes the relaxed objective by
\[
  \epsilon(\sigma_{t_1}-\sigma_{t_2})\tilde a'(\bar K)+O(\epsilon^2)>0
\]
for all sufficiently small \(\epsilon>0\). This proves the strict and first-order assertions. \Halmos
\endproof

\proof{Proof of Proposition~\ref{prop:majorization}.}
The feasible set $\mathcal K:=\{K\in[1,\infty)^T:\sum_t K_t=B\}$ is nonempty ($B\ge T$), compact, and convex, and $\tilde a$ is continuous, so $\tilde V^*$ is attained and finite.

\emph{(i).} Feasibility of $K_t\equiv B/T$ gives $\tilde V^*(\bsigma)\ge\tilde a(B/T)\sum_t\sigma_t$, i.e. $\Delta_{\mathrm{unif}}\ge0$. For each fixed $K\in\mathcal K$, $\bsigma\mapsto\sum_t\sigma_t\tilde a(K_t)$ is linear, so $\tilde V^*$ is a supremum of linear maps, hence convex; subtracting the linear $\bsigma\mapsto\tilde a(B/T)\sum_t\sigma_t$ preserves convexity. As $\mathcal K$ is permutation-invariant, $\tilde V^*(P\bsigma)=\tilde V^*(\bsigma)$ for every permutation $P$ and $\sum_t\sigma_t$ is symmetric, so $\Delta_{\mathrm{unif}}$ is symmetric.

\emph{(ii).} If $\bsigma$ majorizes $\bsigma'$, then $\bsigma'=D\bsigma$ for a doubly stochastic $D$ (Hardy--Littlewood--P\'olya), and $D=\sum_i\lambda_i P_i$ with permutations $P_i$ and $\lambda_i\ge0$, $\sum_i\lambda_i=1$ (Birkhoff) \citep{marshall2011inequalities}. Convexity and symmetry give
\[
  \Delta_{\mathrm{unif}}(\bsigma')=\Delta_{\mathrm{unif}}\Bigl(\sum\nolimits_i\lambda_i P_i\bsigma\Bigr)\le\sum\nolimits_i\lambda_i\,\Delta_{\mathrm{unif}}(P_i\bsigma)=\Delta_{\mathrm{unif}}(\bsigma).
\]

\emph{(iii).} Let $B>T$, so $B/T>1$ is interior. If all $\sigma_t=s\ge0$, strict concavity of $\tilde a$ and Jensen give $\sum_t\tilde a(K_t)\le T\tilde a(B/T)$ on $\mathcal K$, with equality at the uniform point; hence $\tilde V^*(\bsigma)=sT\tilde a(B/T)$ and $\Delta_{\mathrm{unif}}(\bsigma)=0$. Conversely, if the $\sigma_t$ are not all equal, the perturbation in Part (iv) of the proof of Theorem~\ref{thm:offline} shows the uniform point is not a maximizer, so $\tilde V^*(\bsigma)>\tilde a(B/T)\sum_t\sigma_t$ and $\Delta_{\mathrm{unif}}(\bsigma)>0$. \Halmos
\endproof

\proof{Proof of Theorem~\ref{thm:lower}.}
Fix any possibly randomized adaptive policy and let $T=2B$. Both instances set $\cD_t=\delta_{1/2}$ for $t\le B$. Over the second half, instance A sets $\cD_t=\delta_0$, whereas instance B sets $\cD_t=\delta_1$. The two instances therefore induce the same joint law for the history observed through round $B$. Under the fixed policy, the distribution of the first-half expenditure is consequently the same on both instances; denote its common expectation by
\[
  S:=\E\Bigl[\sum_{t=1}^B K_t\Bigr]\in[0,B].
\]
On instance A, the offline benchmark spends all $B$ units in the first half and earns $B/2$, whereas the policy earns $S/2$ because all second-half rewards vanish. On instance B, the offline benchmark spends all $B$ units in the second half and earns $B$, whereas the policy earns at most $S/2+(B-S)$. Thus, writing $\mathrm{Reg}_A$ and $\mathrm{Reg}_B$ for the expected regrets on the two instances,
\[
  \mathrm{Reg}_A\ge\frac{B-S}{2},
  \qquad
  \mathrm{Reg}_B\ge\frac{S}{2}.
\]
It follows that
\[
  \max\{\mathrm{Reg}_A,\mathrm{Reg}_B\}
  \ge
  \max\Bigl\{\tfrac{B-S}{2},\tfrac{S}{2}\Bigr\}
  \ge \frac{B}{4},
\]
where the final expression is minimized at $S=B/2$. Hence at least one of the two instances has expected regret at least $B/4$; since $T=2B$, this is $\Omega(T)$. \Halmos
\endproof

\proof{Proof of Proposition~\ref{prop:transfer}.}
The two profiles coincide through step $B$ and therefore induce the same joint law for the history observed in the first half of the surrogate allocation game. Hence any possibly randomized nonanticipative policy has the same distribution of first-half expenditure on the two profiles; write
\[
  S:=\E\Bigl[\sum_{t=1}^B e_t\Bigr]\in[0,B]
\]
for its common expectation. Concavity of $a_K$ gives
\[
  a_{1+e}-a_1\le e\,\Delta a_1,\qquad e\in\mathbb Z_{\ge0}.
\]
On profile A, the offline optimum assigns one discretionary unit to each of the first $B$ steps and earns $m\Delta a_1B$, whereas the policy earns at most $m\Delta a_1S$. On profile B, the offline optimum assigns one unit to each of the last $B$ steps and earns $2m\Delta a_1B$, whereas the policy earns at most $m\Delta a_1S+2m\Delta a_1(B-S)$. Therefore,
\[
  \mathrm{Reg}_A \ge m\Delta a_1(B-S),
  \qquad
  \mathrm{Reg}_B \ge 2m\Delta a_1 B - m\Delta a_1 S - 2m\Delta a_1(B-S) = m\Delta a_1 S,
\]
and hence
\[
  \max\{\mathrm{Reg}_A,\mathrm{Reg}_B\}
  \ge m\Delta a_1\max\{B-S,S\}
  \ge \frac{m\Delta a_1B}{2}.
\]

For the diffusion transfer, set $g_{\max}:=\max_{1\le t\le2B}g_t$. The $2B$ mandatory candidates and $B$ discretionary candidates give $\sum_tK_t=3B$. Under the uniform approximation bound in Proposition~\ref{prop:transfer}, the deviation of any feasible allocation from its surrogate value is bounded by
\[
  C\sum_{t=1}^{2B}K_tg_t^2h
  \le 3CBg_{\max}^2h.
\]
Applying this bound once to the policy value and once to the corresponding offline benchmark shows that realized regret differs from surrogate regret by at most $6CBg_{\max}^2h$. This is the uniform-in-$t$ bound stated as $O(Bg_{\max}^2h)$ in Proposition~\ref{prop:transfer} and proves the transfer. \Halmos
\endproof

\medskip
\noindent The next construction realizes the approximation bound in Proposition~\ref{prop:transfer} within the abstract model of Sections~\ref{sec:formulation} and~\ref{sec:framework}.

\begin{proposition}\label{prop:realizable} (Exact realizability of the hard instances.)
Fix $m > 0$, $B \ge 1$, $T = 2B$, and a step size $h > 0$. Within the model of Sections~\ref{sec:formulation} and~\ref{sec:framework}, consider the instance class with state space $\R^d$, driftless transitions ($b_t \equiv 0$ in \eqref{eq:em-step}), the identity denoiser $D_\theta(x, s, c) = x$, the linear verifier $v(x, c) = \langle w, x\rangle$ for a fixed unit vector $w$, and the diffusion schedule $g_t := \sigma_t/\sqrt h$, where the profile $\{\sigma_t\}$, indexed in traversal order, equals $\bsigma^A$ or $\bsigma^B$ from Proposition~\ref{prop:transfer}. Then the following statements hold.
\begin{enumerate}
\item[(i)] Assumption~\ref{asm:smooth} holds with $L_t = 0$, and for every $t$ and $K$ the per-step gain satisfies $G_t(K) = \sigma_t\,a_K$ exactly, with zero remainder.
\item[(ii)] For every deterministic allocation $\{K_t\}$, the exact end-to-end objective satisfies $\E[v(x_0, c)] = \E[\langle w, x_T\rangle] + \sum_t \sigma_t\,a_{K_t}$; for policies that select the count at each step adaptively, but \emph{before observing any candidate at that step}, the identity holds with $a_{K_t}$ replaced by $\E[a_{K_t}]$. In particular, the surrogate objective coincides with the exact objective on this class.
\item[(iii)] The two instances induce identical joint laws for all observables over the first $B$ traversed steps. Consequently, every, possibly randomized, policy that, upon reaching each step, commits to the discretionary count before observing that step's candidate scores, even if granted direct observation of $\sigma_t$ on arrival, suffers worst-case \emph{exact} regret at least $m\,\Delta a_1\,B/2$ across the two instances, and the transfer of Theorem~\ref{thm:lower} is unconditional within this class.
\end{enumerate}
\end{proposition}

\proof{Proof of Proposition~\ref{prop:realizable}.}
\emph{(i).} Here $\Psi_t(x;c)=v(D_\theta(x,t{-}h,c);c)=\langle w,x\rangle$ is affine, so $\nabla\Psi_t\equiv w$, $\nabla^2\Psi_t\equiv0$, and Assumption~\ref{asm:smooth} holds with $L_t=0$ and $\sigma_t=g_t\sqrt h\,\|w\|$. The Lemma~\ref{lem:taylor} remainder vanishes, so $S_t^{(k)}=\langle w,X_t\rangle+\sigma_t\langle w,\xi^{(k)}\rangle$ with $\langle w,\xi^{(k)}\rangle\overset{\mathrm{iid}}{\sim}\mathcal N(0,1)$, whence $G_t(K)=\sigma_t\E[\max_{k\le K}\langle w,\xi^{(k)}\rangle]=\sigma_t a_K$ exactly (and $G_t(K)=0=\sigma_t a_K$ when $\sigma_t=0$).

\emph{(ii).} With $b_t\equiv0$,
\[
  x_0=x_T+\sum_t g_t\sqrt h\,\xi_t^*,
  \qquad
  v(x_0,c)=\langle w,x_T\rangle+\sum_t\sigma_t\langle w,\xi_t^*\rangle.
\]
At a step with $\sigma_t>0$, all candidates share the common shift $\langle w,X_t\rangle$, so the retained candidate maximizes $\langle w,\xi^{(k)}\rangle$. Let $\mathcal H_t$ denote the history available upon arrival at step $t$. If $K_t$ is $\mathcal H_t$-measurable and is fixed before the fresh candidate draws, then their conditional law remains iid standard Gaussian and
\[
  \E[\langle w,\xi_t^*\rangle\mid\mathcal H_t,K_t]=a_{K_t}.
\]
For deterministic counts, summing this identity gives
\[
  \E[v(x_0,c)]
  =
  \E[\langle w,x_T\rangle]+\sum_t\sigma_ta_{K_t}.
\]
For pre-committed adaptive counts, the tower property instead gives the same expression with $a_{K_t}$ replaced by $\E[a_{K_t}]$. Thus the surrogate objective coincides with the exact objective on this class. The pre-commitment restriction is substantive: if a rule stops in response to favorable within-step draws, conditioning on the realized count can create optional-stopping bias, and $\E[\langle w,\xi_t^*\rangle\mid K_t]$ need not equal $a_{K_t}$. Part~(iii) excludes such rules, matching the class of Proposition~\ref{prop:transfer}.

\emph{(iii).} Over the first $B$ traversed steps, the two instances share all primitives: the coefficients $g_t$, score maps, and candidate-generation laws are identical. They therefore induce the same joint law for every observable history through those steps, and any fixed nonanticipative policy has the same expected first-half expenditure on both instances. Replacing the surrogate gains in the proof of Proposition~\ref{prop:transfer} by the exact gains from part~(ii) gives worst-case exact regret at least $m\Delta a_1B/2$. Direct observation of $\sigma_t$ on arrival does not distinguish the two instances during the first half, so the same lower bound continues to hold for the enlarged policy class.

For completeness, replace the zero second-half coefficients of instance A by a common $\epsilon\in(0,m)$. If $S$ denotes expected first-half expenditure, the two regret lower bounds become
\[
  (m-\epsilon)\Delta a_1(B-S)
  \qquad\text{and}\qquad
  m\Delta a_1S.
\]
Balancing them yields
\[
  \frac{m(m-\epsilon)}{2m-\epsilon}\,\Delta a_1B
  \ge
  \frac{m\Delta a_1B}{2}-\frac{\epsilon\Delta a_1B}{2},
\]
so the coefficients may be kept strictly positive with only the stated loss. \Halmos
\endproof

\proof{Proof of Proposition~\ref{prop:jensen}.}
As in \eqref{eq:alloc}, the feasible counts are integer-valued. For each fixed feasible $\{K_t\}$, $\bsigma\mapsto\sum_t\sigma_t a_{K_t}$ is linear, so $V^*$ is the pointwise maximum of finitely many such maps and is therefore convex, proving (i). For (ii), linearity of expectation gives $\E[\sum_t\sigma_t a_{K_t}]=\sum_t\bar\sigma_t a_{K_t}$, so the best fixed allocation attains expected value $\max_{\{K_t\}}\sum_t\bar\sigma_t a_{K_t}=V^*(\bar\bsigma)$. The instance-wise oracle attains $\E[V^*(\bsigma)]$, and Jensen applied to the convex $V^*$ gives
\[
  \E[V^*(\bsigma)] \ge V^*(\E\,\bsigma) = V^*(\bar\bsigma),
\]
so the expected shortfall is $\E[V^*(\bsigma)]-V^*(\bar\bsigma)=J(\bsigma)\ge0$. \Halmos
\endproof

\section{General Local Search Operators: Proofs and Instantiations}
\label{app:general}

This appendix supplies the proofs for Section~\ref{sec:general-operators} and instantiates the abstract framework on three concrete operators. Section~\ref{app:idealization} makes precise the projected-score idealization under which the score-dependent operators are analyzed. Section~\ref{app:general-proofs} proves Theorem~\ref{thm:general-gain} and Corollary~\ref{cor:general-alloc} and verifies the iterate moment condition. Section~\ref{app:examples} derives the gain sequences of Table~\ref{tab:phi-summary} and verifies Assumption~\ref{asm:local-search} for each operator.

\subsection{The Projected-Score Idealization}
\label{app:idealization}

Score-dependent operators accept or reject a candidate by comparing verifier scores. When $\nabla\Psi_t(\bar X_{t-h};c)\neq0$, the idealization analyzed below instead resolves the comparison by the linearized projected score $\langle u,\xi\rangle$, where $u:=\nabla\Psi_t(\bar X_{t-h};c)\allowbreak/\allowbreak\|\nabla\Psi_t(\bar X_{t-h};c)\|$ is the gradient direction. The zero-gradient case is handled directly through the second-order remainder in the proof of Theorem~\ref{thm:general-gain}. The following remark quantifies the discrepancy between the two comparison rules in the nondegenerate case.

\begin{remark}\label{rmk:flip} (Acceptance-flip window.)
For a candidate $\eta$ compared against an incumbent $\zeta$, Lemma~\ref{lem:taylor} gives
\begin{align*}
\Psi_t(\bar X_{t-h} + g_t\sqrt h\,\eta;c) - \Psi_t(\bar X_{t-h} + g_t\sqrt h\,\zeta;c) &= \sigma_t\langle u, \eta-\zeta\rangle + \delta_t(\eta,\zeta),\\
|\delta_t(\eta,\zeta)| &\le \tfrac12 L_t g_t^2 h\,\bigl(\|\eta\|^2 + \|\zeta\|^2\bigr).
\end{align*}
Hence the exact and the projected accept-or-reject decisions can disagree only on the event
\begin{equation}
  \bigl|\langle u, \eta-\zeta\rangle\bigr| \;\le\; \frac{L_t\,g_t\sqrt h}{2\,\|\nabla\Psi_t(\bar X_{t-h};c)\|}\,\bigl(\|\eta\|^2 + \|\zeta\|^2\bigr),
  \qquad \sigma_t>0,
  \label{eq:flip-window}
\end{equation}
  a window whose width shrinks linearly in $g_t\sqrt h$. For a sequential operator, couple the exact and projected procedures using the same primitive random draws. Up to their first disagreement the incumbents coincide, so that first disagreement can occur only at a comparison satisfying \eqref{eq:flip-window}; after such a disagreement, subsequent iterates may differ. A probability bound for a first disagreement, and hence for the induced error in the gain sequence, follows under anti-concentration of the projected comparison gap near zero together with a lower bound on $\|\nabla\Psi_t\|$. The $\epsilon$-greedy and local-perturbation results in this paper are exact for the projected-score idealization. For exact score-comparison dynamics, these two quantitative conditions identify the additional ingredients needed to turn the first-disagreement characterization into a probability or gain bound.
\end{remark}

\subsection{Proofs of Theorem~\ref{thm:general-gain} and Corollary~\ref{cor:general-alloc}}
\label{app:general-proofs}

We first record the deduction of condition (A3) from condition (A4), which the proof of Theorem~\ref{thm:general-gain} invokes. Each operator considered in this paper is built from isotropic primitives, namely Gaussian exploration draws and uniformly directed perturbations, and accepts or rejects candidates on the basis of projected scores alone; rotating all iterates by an orthogonal matrix $Q$ while rotating the score direction $u$ to $Qu$ therefore leaves both the law of the primitives and every acceptance event unchanged, since $\langle Qu, Q\xi\rangle = \langle u, \xi\rangle$. This is precisely condition (A4). To deduce (A3), fix a unit vector $u$ and choose an orthogonal $Q$ with $Qu = e_1$; writing $\{\xi^{(j)}_u\}$ for the iterates generated under direction $u$, we have $\langle u, \xi^{(j)}_u\rangle = \langle e_1, Q\xi^{(j)}_u\rangle$, and by (A4) the process $\{Q\xi^{(j)}_u\}$ has the joint law of $\{\xi^{(j)}_{e_1}\}$, so $\max_j \langle u, \xi^{(j)}_u\rangle$ has the same distribution for every $u$, and the projected gain sequence $\phi_K^{\cL}$ is well defined.

\proof{Proof of Theorem~\ref{thm:general-gain}.}
First suppose $\nabla\Psi_t(\bar X_{t-h};c)=0$. Conditional on $(X_t,c)$, Lemma~\ref{lem:taylor} gives
\[
  \Psi_t(\bar X_{t-h}+g_t\sqrt h\,\xi;c)=\mu_t+R_t(\xi),
  \qquad
  |R_t(\xi)|\le\tfrac12L_tg_t^2h\|\xi\|^2.
\]
Since $\sigma_t=0$, the moment assumption and $|\max_j r_j|\le\max_j|r_j|$ imply
\[
  |G_t^{\cL}(K)|
  \le
  \E\Bigl[\max_{1\le j\le K}|R_t(\xi^{(j)})|\,\Big|\,X_t,c\Bigr]
  \le
  \tfrac12L_tg_t^2h\,M_K(\cL,d).
\]
Thus the claimed expansion holds in the degenerate case without introducing an artificial score direction.

Now suppose $\nabla\Psi_t(\bar X_{t-h};c)\neq0$ and set
\[
  u_t:=\frac{\nabla\Psi_t(\bar X_{t-h};c)}
             {\|\nabla\Psi_t(\bar X_{t-h};c)\|}.
\]
Applying Lemma~\ref{lem:taylor} to each iterate gives
\[
  \Psi_t(\bar X_{t-h}+g_t\sqrt h\,\xi;c)
  =
  \mu_t+\sigma_t\langle u_t,\xi\rangle+R_t(\xi),
  \qquad
  |R_t(\xi)|\le\tfrac12L_tg_t^2h\|\xi\|^2.
\]
Because $\mu_t$ is common to all iterates, substitution into \eqref{eq:gain-general} yields
\begin{align*}
  G_t^{\cL}(K) &= \E\Bigl[\max_{1 \le j \le K}\bigl(\sigma_t \langle u_t, \xi^{(j)}\rangle + R_t(\xi^{(j)})\bigr)\,\Big|\,X_t, c\Bigr] \\
                &= \sigma_t\,\E\Bigl[\max_{1 \le j \le K} \langle u_t, \xi^{(j)}\rangle\,\Big|\,X_t, c\Bigr] + O_{K,d,L_t,\cL}(g_t^2 h),
\end{align*}
where the second equality uses the inequality $|\max_j(a_j + b_j) - \max_j a_j| \le \max_j |b_j|$, the non-negativity $\sigma_t = g_t\sqrt h\,\|\nabla\Psi_t\| \ge 0$ to factor $\sigma_t$ out of the maximum, and the moment-bounded remainder
\[
  \E\Bigl[\max_{1 \le j \le K} |R_t(\xi^{(j)})|\,\Big|\,X_t,c\Bigr]
  \le \tfrac{1}{2}L_t g_t^2 h\,
      \E\Bigl[\max_{1 \le j \le K}\|\xi^{(j)}\|^2\,\Big|\,X_t,c\Bigr]
  \le \tfrac{1}{2}L_t g_t^2 h\,M_K(\cL,d).
\]
The final inequality is the posited moment bound, interpreted conditionally on the fixed current state. By condition (A4) and the deduction preceding this proof, the conditional projected expectation does not depend on $u_t$ and equals $\phi_K^{\cL}$ by definition. \Halmos
\endproof

We now verify the moment condition assumed in Theorem~\ref{thm:general-gain} for the three operators considered in this paper. In the statements below, $\lambda$ denotes the perturbation radius scale of the $\epsilon$-greedy and local-perturbation operators, whose dynamics are specified in Section~\ref{app:examples}.

\begin{lemma}\label{lem:moments} (Iterate moment bounds.)
The moment condition of Theorem~\ref{thm:general-gain} holds for the three operators of Table~\ref{tab:phi-summary}, uniformly in $(c, x_t)$:
\[
  M_K^{\mathrm{RS}} \le K d, \qquad
  M_K^{\mathrm{LP}} \le 2d + 2(K-1)^2\lambda^2 d, \qquad
  M_K^{(\epsilon)} \le 2Kd + 2(K-1)^2\lambda^2 d.
\]
\end{lemma}

\proof{Proof of Lemma~\ref{lem:moments}.}
\emph{Random search:} $\E[\max_{j\le K}\|\xi^{(j)}\|^2]\le\sum_{j\le K}\E\|\xi^{(j)}\|^2=Kd$. \emph{Local perturbation:} each candidate is $\xi^{(1)}$ plus at most $K{-}1$ increments of norm $\le\lambda\sqrt d$, so $\max_j\|\xi^{(j)}\|\le\|\xi^{(1)}\|+(K{-}1)\lambda\sqrt d$, and $(a+b)^2\le2a^2+2b^2$ with $\E\|\xi^{(1)}\|^2=d$ gives $2d+2(K{-}1)^2\lambda^2 d$. \emph{$\epsilon$-greedy:} each candidate is one of $\le K$ Gaussian seeds $\{g_i\}$ plus $\le K{-}1$ increments, so $\max_j\|\xi^{(j)}\|\le\max_i\|g_i\|+(K{-}1)\lambda\sqrt d$ with $\E[\max_i\|g_i\|^2]\le Kd$, giving $2Kd+2(K{-}1)^2\lambda^2 d$. None of these bounds depends on $(c,x_t)$, which gives the state-uniform remainder control used in Theorem~\ref{thm:general-gain}. \Halmos
\endproof

\proof{Proof of Corollary~\ref{cor:general-alloc}.}
Instance-independence is as in Theorem~\ref{thm:offline}(i): with $\sigma_t$ depending only on $t$, $\sum_t\sigma_t\phi^{\cL}_{K_t}$ is a fixed function of $\{K_t\}$. By (A3), the gain sequence $\phi_K^{\cL}$ is common across steps, and by (A2) its marginals are non-increasing. The one-unit exchange argument in the proof of Theorem~\ref{thm:offline} therefore applies with $a_K$ replaced by $\phi_K^{\cL}$, including its separate treatment of the boundary case $B=T$. It yields the common marginal-threshold conditions \eqref{eq:marginal-threshold} with $\Delta a$ replaced by $\Delta\phi^{\cL}$ and the same weak monotonicity in sensitivity. If $\phi_K^{\cL}$ is strictly concave, the active-tie exclusion and feasible-uniform dominance arguments of Theorem~\ref{thm:offline} apply with $\Delta a$ replaced by $\Delta\phi^{\cL}$. If $\phi_K^{\cL}=\alpha(K-\beta)$ is affine ($\alpha>0$), then $\sum_t\sigma_t\phi^{\cL}_{K_t}=\alpha\sum_t\sigma_t K_t+\mathrm{const}$, an integer allocation problem with an affine objective whose optimizers place all budget above the floor $K_t\ge1$ on the maximal-$\sigma$ steps (ties broken arbitrarily). Finally, for each fixed integer allocation, $\bsigma\mapsto\sum_t\sigma_t\phi^{\cL}_{K_t}$ is linear. Hence $V^*_{\cL}$ is the maximum of finitely many linear maps and is convex, so the proof of Proposition~\ref{prop:jensen} applies with $a_K$ replaced by $\phi_K^{\cL}$. \Halmos
\endproof

\subsection{Operator-Specific Gain Sequences}
\label{app:examples}

We now instantiate the framework on the three operators of Table~\ref{tab:phi-summary}. All derivations are conducted in the projected-score idealization of Section~\ref{app:idealization}, which replaces the exact score by the linearized score $\mu_t+\sigma_t\langle u,\xi\rangle$. The omitted Taylor remainder obeys the pathwise bound $|R_t(\xi)|\le\tfrac12L_tg_t^2h\|\xi\|^2$, and Theorem~\ref{thm:general-gain} controls its expected maximum through the operator-specific moment bound. The derived sequences are therefore exact for the idealization and describe the leading-order behavior covered by that theorem. Throughout, $d\ge1$; for the $\epsilon$-greedy operator we take a fixed $\epsilon\in(0,1]$ and $\lambda\ge0$, while pure local perturbation uses $\lambda>0$.

\emph{Random search.} The operator draws $K$ independent candidates $\xi^{(j)} \sim \mathcal{N}(0, I_d)$ and retains the best. The projected scores $\langle u, \xi^{(j)}\rangle$ are independent standard Gaussians, so $\phi_K^{\mathrm{RS}} = a_K$. Condition (A1) holds because a fresh draw exceeds the running maximum with positive probability. For condition (A2), write $\Delta_K = \E[h(M_K)]$ with $h(m) := \E[(Z-m)_+]$ for $Z \sim \mathcal N(0,1)$ and $M_K := \max(Z_1, \ldots, Z_K)$; the function $h$ is strictly decreasing, since $h'(m) = -(1-\Phi(m)) < 0$, and $M_{K+1} \ge M_K$ almost surely with strict inequality on an event of positive probability, so $\Delta_{K+1} < \Delta_K$ for all $K \ge 1$, which is the classical concavity of $a_K$ \citep[Chapter~4]{david2003order}. Conditions (A3) and (A4) are immediate because independent isotropic Gaussian draws ignore the score direction entirely.

\emph{The $\epsilon$-greedy operator.} The first iterate is drawn $\xi^{(1)} \sim \mathcal{N}(0, I_d)$ and becomes the incumbent. For $j \ge 2$, with probability $\epsilon$ the operator \emph{explores} by drawing a fresh $\xi^{(j)} \sim \mathcal{N}(0, I_d)$; otherwise it \emph{exploits} by proposing $\xi^{(j)} = \xi^* + RU$, a perturbation of the incumbent $\xi^*$, where $U$ is uniform on the unit sphere, $R \sim \mathrm{Unif}[0, \lambda\sqrt{d}]$, and the incumbent is updated whenever the proposal scores higher. In the idealization, acceptance is resolved by the projected score, and the incumbent's projection coincides with the projected running maximum $M_K := \max_{j \le K}\langle u, \xi^{(j)}\rangle$. The marginal gain then decomposes as
\begin{equation}
  \Delta_K \;=\; \epsilon \cdot \E[h(M_K)] \;+\; (1-\epsilon)\,c_d(\lambda),
  \label{eq:eps-marginal}
\end{equation}
since an exploration step contributes a fresh standard Gaussian independent of $M_K$, while an exploitation step adds the accepted increment $(R\langle u, U\rangle)_+$, which is independent of the incumbent with expectation
\begin{equation}
    c_d(\lambda) := \E\bigl[(R\langle u, U\rangle)_+\bigr] = \frac{\lambda\sqrt{d}}{4\sqrt{\pi}}\,\frac{\Gamma(d/2)}{\Gamma((d+1)/2)}.
    \label{eq:uniform-ball-drift}
\end{equation}
To verify \eqref{eq:uniform-ball-drift}, write $V = \langle u, U\rangle$ and use independence: $\E[R] = \lambda\sqrt d/2$, and for $d \ge 2$ the projection $V$ has density proportional to $(1-v^2)^{(d-3)/2}$ on $(-1,1)$, so that
\[
  \E[V_+] = \frac{\Gamma(d/2)}{\sqrt\pi\,\Gamma((d-1)/2)}\int_0^1 v(1-v^2)^{(d-3)/2}\,dv
  = \frac{\Gamma(d/2)}{\sqrt\pi\,\Gamma((d-1)/2)}\cdot\frac{1}{d-1}
  = \frac{\Gamma(d/2)}{2\sqrt\pi\,\Gamma((d+1)/2)},
\]
using the substitution $s = v^2$ and the identity $\Gamma((d+1)/2) = \tfrac{d-1}{2}\Gamma((d-1)/2)$; multiplying by $\E[R]$ yields \eqref{eq:uniform-ball-drift}, and the case $d = 1$, in which $V$ is a symmetric sign, gives $c_1(\lambda) = \lambda/4$, consistent with the same formula. Counting exploration steps by $N \sim \mathrm{Bin}(K-1, \epsilon)$ and conditioning on $N$ yields the lower bound $\phi_K^{(\epsilon)} \ge \E[a_{N+1}]$ reported in Table~\ref{tab:phi-summary}. Equality holds at $\lambda=0$ (and in the limit as $\lambda\to0$), because exploitation then contributes no projected improvement; it also holds at $\epsilon=1$, when the operator reduces to random search. For fixed $\epsilon>0$, a Chernoff bound gives $\Pr\{N\ge\epsilon(K-1)/2\}\to1$, and monotonicity together with $a_n\sim\sqrt{2\log n}$ yields
\[
  \E[a_{N+1}]
  \ge
  \Pr\{N\ge\epsilon(K-1)/2\}\,
  a_{\lfloor\epsilon(K-1)/2\rfloor+1}
  =
  \Omega(\sqrt{\log K}),
\]
which justifies the growth-rate entry in Table~\ref{tab:phi-summary}. Condition (A1) holds since each iteration explores with probability $\epsilon>0$ and a Gaussian draw exceeds any incumbent with positive probability. For condition (A2), the exploitation term in \eqref{eq:eps-marginal} is constant in $K$, while $\E[h(M_{K+1})] < \E[h(M_K)]$ by the argument used for random search, so $\Delta_{K+1} \le \Delta_K$, strictly when $\epsilon > 0$. Conditions (A3) and (A4) hold because both primitives are isotropic and the acceptance comparison depends on the candidates only through their projections, so jointly rotating iterates and score direction preserves the law.

\emph{Local perturbation search.} This operator is the pure-exploitation limit $\epsilon = 0$: every iteration perturbs the incumbent by $RU$ and accepts if the proposal scores higher. In the idealization, the perturbation is accepted precisely when $\langle u, U\rangle > 0$, an event independent of the incumbent, so each iteration beyond the first adds the constant expected increment $c_d(\lambda)$, yielding the exactly affine sequence
\begin{equation}
  \phi_K^{\mathrm{LP}} = c_d(\lambda)\,(K-1), \qquad K \ge 1,
  \label{eq:phi-lp-exact}
\end{equation}
with $\phi_1^{\mathrm{LP}} = 0$ because the initial draw has zero-mean projection. Condition (A1) holds for $\lambda > 0$ and $d \ge 1$, since then $c_d(\lambda) > 0$ and each perturbation improves the incumbent with probability one half. Condition (A2) holds with equality throughout, which is the affine case of Corollary~\ref{cor:general-alloc}: the optimal fixed-profile projected-score surrogate schedule concentrates the discretionary budget on the most sensitive steps. Conditions (A3) and (A4) follow as in the $\epsilon$-greedy case. For exact score comparisons, curvature effects modify the constant-marginal structure; equation~\eqref{eq:phi-lp-exact} gives the leading-order gain sequence.

The three operators differ qualitatively in growth rate, as summarized in Table~\ref{tab:phi-summary}: random search grows logarithmically while local perturbation grows linearly in the idealization. The preferred operator therefore depends on the per-step budget, with random search favored at small budgets and perturbation search at large ones. A next step is to quantify the crossover for exact comparisons by controlling how the second-order remainder interacts with operator-dependent iterate norms as perturbation iterates accumulate drift. Within the projected-score regime, the ordering in Table~\ref{tab:phi-summary} gives the operator-selection rule.

\section{A Markov Decision Process Formulation}
\label{app:mdp}

This appendix serves two purposes. First, it presents a general Markov decision process (MDP) for inference-time search over noise trajectories, lifting the restrictions of the formulation analyzed in the main text. Second, it uses the MDP as a common language for classifying ten existing search procedures, together with \GAINS{}, as policies within a single decision model (Table~\ref{tab:mdp_policies}).

The main-text formulation restricts the policy class in two respects: the sampler advances through the steps in a single forward pass, and every step applies the same local search operator to its noise variable. The MDP introduced here removes both restrictions, allowing multiple verifiers of heterogeneous cost, backtracking to earlier steps, and step-dependent choices of search operator. \citet{ramesh2025nts} first cast diffusion sampling as an MDP with state $(c, x_t, t)$ and terminal reward $v(x_0, c)$, and then relaxed the formulation to independent contextual bandits for tractability. We extend their formulation in four respects: the state carries the search history and the remaining budget; the action set is partitioned into search, verification, and movement categories; the reward admits weighted combinations of multiple verifiers; and we use the resulting MDP to classify existing procedures as special-case policies.

\emph{State space.} At each decision epoch, the policy observes $s = (c, t, x_t, \cH_t, R)$, where $c$ is the conditioning signal, $t\in\{1,\dots,T\}$ indexes the current step, $x_t$ is the current latent, $\cH_t$ records the search history at the current step (past candidates, their scores, and the actions taken), and $R\in\{0,\dots,B\}$ is the remaining budget in function evaluations. The initial state is $s_0=(c, T, x_T, \emptyset, B)$ with $x_T\sim\mathcal{N}(0,I_d)$; the episode terminates when $t=0$, at which point a clean sample has been produced, or when $R=0$.

\emph{Action space.} The actions partition into three categories, $\cA(s) = \cA_{\text{search}}\cup\cA_{\text{verify}}\cup\cA_{\text{move}}$, according to whether they propose a candidate, score the current state, or commit to a transition.
\begin{itemize}
  \item \emph{Search actions} perform one iteration of a local search operator on the noise variable at the current step, leaving $t$ and $x_t$ unchanged. Variants include random search (draw a fresh Gaussian candidate and retain it if it scores higher than the incumbent), $\epsilon$-greedy search (explore with a fresh draw with probability $\epsilon$, otherwise perturb the incumbent), zero-order search (a finite-difference gradient step on the noise variable), and Langevin search (one Markov chain Monte Carlo step). Each iteration costs one function evaluation.
  \item \emph{Verification actions} evaluate the current latent or its one-step clean-sample prediction without modifying the state. A light verification reuses the already computed prediction and applies an inexpensive verifier at no additional function evaluations; a heavy verification expends one or more additional passes through the denoising network to refine the prediction before scoring.
  \item \emph{Movement actions} commit to a step transition. A forward move sets $x_{t-1}=F_\theta(x_t,t,c,\varepsilon_t^*)$ using the best candidate found at the current step and decrements $t$; it is charged zero function evaluations because the pass through the network was already expended when the committed candidate was scored, so that a step at which $K_t$ search actions are taken consumes exactly $K_t$ evaluations, matching the accounting of the main text. A backtrack move returns from $t$ to $t+1$, restores the previously saved latent, retains the search history, and likewise costs zero function evaluations.
\end{itemize}

\emph{Transitions and reward.} Transitions follow directly from the action descriptions: a search action augments $\cH_t$ and decrements $R$; a forward move decrements $t$ and leaves $R$ unchanged; a backtrack move increments $t$ and restores the saved latent. The reward is collected only at termination and equals $v(x_0,c)$, where weighted combinations of several verifiers are admissible. If the budget is exhausted before $t=0$, no clean sample exists and the terminal reward is zero; feasible policies reserve one evaluation per remaining step. The objective is to maximize the expected terminal reward subject to the budget.

\begin{table}[t]
\centering
\caption{Existing inference-time search procedures as policies in the noise-trajectory-search MDP. The second column lists the action types each procedure invokes (Fwd equals forward move; Rand, EpsG, Zero, and Lang denote random, $\epsilon$-greedy, zero-order, and Langevin search; Any denotes any operator satisfying Assumption~\ref{asm:local-search}). The third column indicates whether the policy ever revisits an earlier step. The fourth column describes how the per-step search budget is set.}
\label{tab:mdp_policies}
\small
\begin{tabular}{lccl}
\toprule
\textbf{Procedure} & \textbf{Actions used} & \textbf{Backtracks?} & \textbf{Schedule} \\
\midrule
Plain sampling & Fwd & No & $K_t=1$ for all $t$ \\
Best-of-$N$ \citep{ma2024inference} & Fwd & No & $N$ full trajectories \\
$\epsilon$-greedy \citep{ramesh2025nts} & EpsG, Fwd & No & Uniform $K_t=K$ \\
Direct noise optimization \citep{tang2024dno} & Zero, Fwd & No & Uniform $K_t=K$ \\
Rollover budget forcing \citep{kim2025rbf} & Rand, Fwd & No & Online rollover \\
T-SCEND \citep{zhang2025tscend} & Rand, Fwd & No & Fixed two-phase \\
Cyclic diffusion search \citep{lee2025abcd} & Rand, Fwd & Yes & Online adaptive \\
Tree search \citep{zhang2025classical} & Lang, Fwd & Yes & Tree expansion \\
Feynman-Kac steering \citep{singhal2025fk} & Rand, Fwd & No & Non-uniform (particle) \\
Verifier threshold \citep{sundaresha2025vt} & Rand, Fwd & No & Online threshold \\
\midrule
\GAINS{} (this paper) & Any, Fwd & No & Offline plus online \\
\bottomrule
\end{tabular}
\end{table}

Table~\ref{tab:mdp_policies} identifies three extension axes. Existing procedures commit in advance to either an inexpensive or an expensive verifier, while state-dependent switching would add a new action dimension. Backtracking appears in the cyclic and tree-search procedures. Operator selection is static in the listed methods, with one search operator used throughout a trajectory.

The two-level formulation of Section~\ref{sec:framework}, together with Algorithm~\ref{alg:gains}, defines a factored policy within this MDP: it uses one operator and one verifier, allocates the total budget by offline profiling, and adjusts expenditure online within a small slack without backtracking. In MDP terms, the policy's dependence on the history collapses to two running summaries, the recent mean gain and the recent within-step variability, which are exactly the statistics monitored by the online controller of Section~\ref{sec:online}. Future work should determine (i) whether choosing light versus heavy verification from the current state and remaining budget improves terminal verifier score at a fixed number of function evaluations, (ii) when changing the search operator across steps or backtracking increases that score, and (iii) which exact or approximate planning method solves the resulting finite-horizon MDP.


\bibliographystyle{plainnat}
\bibliography{ref}

\end{document}